\documentclass[11pt]{article}

\usepackage[margin=1in]{geometry}
\usepackage[T1]{fontenc}
\usepackage[utf8]{inputenc}
\usepackage{lmodern}
\usepackage{microtype}
\usepackage{amsmath,amssymb,amsthm}
\usepackage{booktabs,tabularx,array}
\usepackage{enumitem}
\usepackage{xcolor}
\usepackage{xurl}
\usepackage[numbers,sort&compress]{natbib}
\usepackage{graphicx}
\usepackage{float}
\usepackage{tikz}
\usepackage[section]{placeins}
\usepackage{hyperref}
\usetikzlibrary{arrows.meta,positioning}
\definecolor{zgInk}{HTML}{34424B}
\definecolor{zgLine}{HTML}{64717A}
\definecolor{zgBlue}{HTML}{ECF0FF}
\definecolor{zgMint}{HTML}{EAF5E7}
\definecolor{zgCream}{HTML}{FFF7D8}
\definecolor{zgPeach}{HTML}{FFF0E0}
\definecolor{zgRose}{HTML}{FCE9E9}
\definecolor{zgGray}{HTML}{F4F3F0}
\tikzset{
  zg node/.style={draw=zgLine,text=zgInk,fill=zgGray,
    line width=0.55pt,rounded corners=5pt},
  zg flow/.style={-{Latex[length=1.9mm,width=1.3mm]},
    draw=zgInk,line width=0.65pt,rounded corners=5pt},
  zg label/.style={text=zgInk}
}

\hypersetup{colorlinks=true,linkcolor=blue!50!black,
  citecolor=blue!50!black,urlcolor=blue!50!black,
  bookmarksnumbered=true,
  pdftitle={ZeroGate: Trust-Preserving Fast Paths for Governed AI Agent Runtimes},pdfauthor={Zexun Wang},
  pdfsubject={Full research report: design, implementation, evidence and deployment},
  pdfkeywords={agent governance, action pass, runtime security, CAVA, PCAA, ZeroGate}}
\setlist[itemize]{leftmargin=1.4em,itemsep=0.2em}
\setlist[enumerate]{leftmargin=1.6em,itemsep=0.2em}
\newcolumntype{L}[1]{>{\raggedright\arraybackslash}p{#1}}
\newcolumntype{Y}{>{\raggedright\arraybackslash}X}
\newtheorem{definition}{Definition}
\newtheorem{proposition}{Proposition}
\newcommand{\zerogate}{\textnormal{\textsc{ZeroGate}}}
\newcommand{\actionpass}{\textnormal{\textsc{ActionPass}}}
\newcommand{\pcaa}{\textnormal{\textsc{PCAA}}}
\newcommand{\cava}{\textnormal{\textsc{CAVA}}}
\newcommand{\SJSON}{\operatorname{SJSON}}
\newcommand{\Execute}{\ensuremath{\mathsf{execute}}}
\newcommand{\Review}{\ensuremath{\mathsf{review}}}
\newcommand{\Deny}{\ensuremath{\mathsf{deny}}}

\newcommand{\PlanCloudTrials}{5}
\newcommand{\PlanCloudPerCell}{160}
\newcommand{\PlanCloudModes}{2}
\newcommand{\PlanCloudLevels}{3}
\newcommand{\PlanCloudConcurrency}{1, 8, and 32}
\newcommand{\PlanCloudAttempts}{4,800}
\newcommand{\PlanSemanticScenarios}{73}
\newcommand{\PlanSemanticDomains}{4}
\newcommand{\PlanSemanticUnique}{292}
\newcommand{\PlanSemanticRepeats}{5}
\newcommand{\PlanSemanticInputs}{1,460}
\newcommand{\PlanSemanticMethods}{7}

\newcommand{\CloudRecords}{4800}
\newcommand{\CloudSuccesses}{4800}
\newcommand{\CloudFailures}{0}
\newcommand{\SemanticCases}{1460}
\newcommand{\SemanticEvaluations}{10220}
\newcommand{\SemanticUniqueCases}{292}
\newcommand{\FaultCases}{10}

\newcommand{\FastPninetyfiveLow}{9.802}
\newcommand{\FastPninetyfiveHigh}{11.374}
\newcommand{\SlowPninetyfiveLow}{25.018}
\newcommand{\SlowPninetyfiveHigh}{334.000}
\newcommand{\CloudResultsTable}{
\begin{tabular}{@{}rlrrrrr@{}}
\toprule
C & Mode & Success/attempts & Boundary p50 & p95 & p99 & Lifecycle p95 \\
\midrule
1 & Synchronous & 800/800 & 14.513 & 25.018 & 46.319 & 71.382 \\
1 & Prepared & 800/800 & 7.105 & 11.374 & 22.350 & 6976.221 \\
8 & Synchronous & 800/800 & 20.593 & 45.379 & 118.640 & 219.112 \\
8 & Prepared & 800/800 & 6.595 & 10.303 & 44.779 & 1860.231 \\
32 & Synchronous & 800/800 & 73.941 & 334.000 & 664.408 & 749.202 \\
32 & Prepared & 800/800 & 6.490 & 9.802 & 22.897 & 2251.117 \\
\bottomrule
\end{tabular}
}
\newcommand{\SemanticResultsTable}{
\begin{tabular}{@{}lrrr@{}}
\toprule
Method & False allow & False not-allow & Errors \\
\midrule
Action Pass evaluator & 0/1360 & 0/100 & 0 \\
Complete fixture policy & 0/1360 & 0/100 & 0 \\
Policy without action binding & 140/1360 & 0/100 & 0 \\
Policy without freshness & 240/1360 & 0/100 & 0 \\
Policy without receipt evidence & 80/1360 & 0/100 & 0 \\
Exact-action equality only & 1200/1360 & 0/100 & 0 \\
Operation/resource membership only & 1300/1360 & 0/100 & 0 \\
\bottomrule
\end{tabular}
}
\newcommand{\FaultResultsTable}{
\begin{tabular}{@{}ll@{}}
\toprule
Experiment & Recorded outcome \\
\midrule
fresh-success & execute; write/readback HTTP 200 \\
payload-mutation & deny; absent (HTTP 404) \\
policy-refresh & review before refresh; execute after \\
expiry & deny; absent (HTTP 404) \\
missing-authority & review; absent (HTTP 404) \\
bad-signature & deny; absent (HTTP 404) \\
same-pass-race & 1/8 admitted \\
distinct-pass-quota-race & 3/8 admitted \\
crash-after-commit & SIGKILL; retry deny; recovery HTTP 200 \\
discarded-ack-retry & Retry HTTP 200; same ETag: true \\
\bottomrule
\end{tabular}
}
\newcommand{\CloudIntervalTable}{
\begin{tabular}{@{}rlrrr@{}}
\toprule
C & Mode & Mean boundary & 95\% lower & Upper \\
\midrule
1 & Synchronous & 16.701 & 16.198 & 17.369 \\
1 & Prepared & 7.853 & 7.724 & 8.020 \\
8 & Synchronous & 24.536 & 23.524 & 25.635 \\
8 & Prepared & 7.846 & 7.610 & 8.044 \\
32 & Synchronous & 114.896 & 102.916 & 133.612 \\
32 & Prepared & 7.936 & 7.805 & 8.068 \\
\bottomrule
\end{tabular}
}

\newcommand{\ReportCaseId}{\path{98996ae5-prepared-8-0-0}}
\newcommand{\ReportCasePassId}{\path{ap_498410044171c1fbf72eb7dc}}
\newcommand{\ReportCasePayloadHash}{\path{cc47ccc0f6ebf40eb2e4b406f7d3aa278a267466e7e6e9074450ec2b345881b6}}
\newcommand{\ReportCaseFingerprint}{\path{157e3da4fb985a81b073d5d4f608f87c5e4ce7c6783b9833221e6c0d00114cc5}}
\newcommand{\ReportCasePassHash}{\path{1cdbada39c5b2882d0de67e22c85ef6ad505512e2cd0727f36b089e03fc58d20}}
\newcommand{\ReportCaseReceiptHash}{\path{df010b04ed923d5987b5f46310392a5a9fee0bf8d5f514aed725f6b2920ee626}}
\newcommand{\ReportCaseETag}{\path{0x8DF1807D61D9533}}
\newcommand{\ReportComparisonTable}{
\begin{tabular}{@{}rrrr@{}}
\toprule
Concurrency & Sync. boundary p95 & Prepared p95 & Reduction \\
\midrule
1 & 25.018 & 11.374 & 54.5\% \\
8 & 45.379 & 10.303 & 77.3\% \\
32 & 334.000 & 9.802 & 97.1\% \\
\bottomrule
\end{tabular}
}
\newcommand{\ReportLifecycleTable}{
\begin{tabular}{@{}rrr@{}}
\toprule
Concurrency & Sync. mean lifecycle & Prepared mean lifecycle \\
\midrule
1 & 53.449 & 4257.126 \\
8 & 130.610 & 1577.035 \\
32 & 483.925 & 2012.902 \\
\bottomrule
\end{tabular}
}
\newcommand{\ReportCaseTimingTable}{
\begin{tabular}{@{}lr@{}}
\toprule
Recorded interval & Milliseconds \\
\midrule
Review and mint & 7.092 \\
Dwell and otherwise unassigned harness time & 1161.897 \\
Worker admission to dispatch & 6.892 \\
Effect request & 103.018 \\
Explicit service verification & 16.314 \\
Complete lifecycle & 1295.214 \\
\bottomrule
\end{tabular}
}

\title{ZeroGate:\\
  Trust-Preserving Fast Paths for Governed AI Agent Runtimes}
\author{Zexun Wang\thanks{Correspondence: \texttt{jw@nd.im}}\\Ond Holdings Inc.}
\date{September 2026\\[0.4em]\small Full Research Report and Engineering Design Study}

\begin{document}
\maketitle

\begin{abstract}
Moving authorization earlier can shorten an agent's dispatch boundary without
removing authorization work. It can also admit an action whose payload,
authority, or relevant state has changed. \zerogate{} separates exact-action
approval from durable local admission: an issuer signs a short-lived
\actionpass{}, and a trusted runtime adapter reconstructs the final action
before a local gate checks its binding and consumes its nonce. A SQLite
transaction couples nonce consumption, applicable quota updates, and an
admission receipt. We state a conditional decision-preservation proposition:
successful local admission implies that a specified synchronous policy would
authorize the same action at the admission point, provided approval is sound,
all policy dependencies are represented and current, observations are
faithful, and consumption is atomic. The implementation alone establishes
neither current-world freshness nor exactly-once remote effects. Evaluation
separates authored semantic fixtures, controlled concurrency and crash
experiments, and an Azure Blob study comparing synchronous and prepared
execution through the same issuer and gate. Both modes mint an exact-action
pass; lifecycle latency includes preparation and prepared-batch dwell.
Across \CloudRecords{} cloud attempts, prepared worker-admission-to-dispatch
p95 ranges from \FastPninetyfiveLow{} to \FastPninetyfiveHigh{}\,ms, versus
\SlowPninetyfiveLow{} to \SlowPninetyfiveHigh{}\,ms synchronously, across the
tested concurrency levels. Prepared mean complete lifecycle is longer at
every level: the boundary improvement is not a net speedup.
The contribution is an explicit revalidation contract, a durable reference
boundary, and an auditable comparison of where authorization cost is paid,
not a new cryptographic primitive or a universal performance frontier.
\end{abstract}

\clearpage
\section*{Executive Summary and Reading Guide}
\addcontentsline{toc}{section}{Executive Summary and Reading Guide}

An agent that has already prepared a well-defined action should not have to
rediscover all of its authorization context at the last possible moment.
However, carrying an earlier approval is useful only if the runtime can tell
whether that approval still applies. A changed recipient, an exhausted budget,
an expired decision, or a different request body must not slip through simply
because the agent can present a familiar identifier. This report examines
how to make that final check small without making its obligations weaker.
The problem becomes important when agents move beyond drafting text into
changing external systems: creating an object, sending a message or changing
a deployment. Permission to use the tool is not approval for every possible
request sent through it.

OSuite is the broader governance platform; \zerogate{} is the admission
design studied here, and an \actionpass{} is the signed object it checks.
Admission means permission to release one request, not confirmation that
the request succeeded. The pass records a short-lived approval for a
concrete action, together with its authority limits and evidence conditions.
The studied pass is exact-action and single-use, not permission to make
arbitrary calls of the same type. An integration component, the adapter,
reconstructs the final outgoing request. The gate checks it and, when
allowing execution, records the decision and marks a unique identifier
(the nonce) as used in one local database transaction.
Preparing a pass earlier changes when work is paid for; it does not remove
the review, the durable write, or the external service's own authorization.

\paragraph{A concrete reason to prepare earlier.}
Consider an illustrative scheduled upload. A team finishes a report before
its release time, so the final file, destination and acting account are
already known. An approval service can review that exact request and sign
its pass while the application waits for release. At release, the gate
checks that the request still matches, the approval is still valid and its
single use remains available. Only then does it permit the upload attempt.
Changing the file or destination requires a new approval; a used or expired
pass cannot be reused. The opportunity is a shorter last step when work can
be completed during an existing wait. If the task needs immediate execution,
preparation may simply add visible delay. This scenario explains the design
opportunity; the measured workload below is a controlled cloud experiment,
not a study of operators using a scheduled-upload product.

\paragraph{Where the inspiration comes from.}
Sony's FeliCa is the contactless card technology; JR East's Suica is a
ticketing system built around such interactions. The useful lesson is not
that agents should imitate a radio protocol, but that a brief local
interaction can sit inside a larger preparation, state-management and
reconciliation process. \zerogate{} applies that separation to approved
actions. Section~\ref{foundation-design-origin} traces the published sources,
including polling, collision handling and interruption behavior, and explains
the separate Shinkansen control-engineering analogy. None supplies a measured
speed or safety guarantee for this software.

\paragraph{What the evidence currently supports.}
In an author-operated Azure deployment, \CloudRecords{} file-like object
writes completed, and separately retrieved bytes matched the approved content
commitments. Table~\ref{tab:paired-tradeoff} compares two timing modes of the
same research implementation, not earlier Studio releases or competitors.
Both paths use the same approval service, gate and storage service.
Synchronous mode requests approval after a dispatch worker receives the
action; prepared mode has already obtained it. The short \emph{boundary}
clock runs from that worker handoff until permission to dispatch is returned.
It excludes the upload itself and waiting to acquire a worker.

Here, p95 is the 95th percentile: at least 95\% of recorded successful
attempts have an interval at or below that value. It is neither an accuracy
percentage nor a response-time guarantee. The longer \emph{lifecycle} clock
starts when preparation begins for that action and ends after service
verification, including the intervening wait and external operation.

\begin{table}[H]
\centering\small
\ReportComparisonTable
\par\vspace{1em}
\ReportLifecycleTable
\caption{Boundary benefit and lifecycle cost from the same cloud records.
All times are milliseconds. Concurrency is the configured worker limit,
not the number of independent agents or arrival rate. The upper panel
compares p95; the lower panel compares means.}
\label{tab:paired-tradeoff}
\end{table}

The prepared batch has a longer mean lifecycle at every level: its preparation
and batch wait are still paid for. These results support a shorter final
boundary, not a faster complete workflow or a measured tap-and-go user
experience. Controlled fault histories make replay,
quota contention and recovery observable. Authored semantic fixtures check
selected authorization obligations, rather than estimating security accuracy
over an independently labeled population.

\paragraph{What this means for a deployment decision.}
The useful question is whether a particular workflow knows its final action
early enough to prepare authorization during otherwise unavoidable time.
For an export whose destination and bytes are already fixed, a pass may let
the last execution step avoid another round trip to the reviewer. For a
conversation whose next tool arguments are not yet known, the same pass
cannot safely authorize an imagined future request. A useful pilot therefore
evaluates preparation usefulness, state freshness, discarded work and
queue-inclusive latency together. The commercial opportunity is to reduce
avoidable last-step waiting while retaining a specific, reviewable link
between approval and execution. Whether that saves time or operating cost
for a customer remains a workflow-specific question.

\paragraph{Terms used throughout.}
An \emph{action} is a concrete operation and its relevant arguments; the
\emph{runtime} is the software environment that carries it out. A
\emph{canonical action} is a structured representation of that request.
Its \emph{fingerprint} is a cryptographic digest used to check whether the
represented details match, not an AI judgment that the action is safe.
An \emph{issuer} reviews and signs the pass; the \emph{gate} checks it at
the release point. An \emph{admission receipt} records that decision,
whereas \emph{effect evidence} records what happened at the destination.
Finally, \emph{freshness} asks whether the facts supporting an approval are
still current; matching two old hashes does not answer that question.

\clearpage
\paragraph{How to read this report.}
The document is a full technical report: its design rationale, mathematical
conditions, measured results and operational implications are intended to
be readable separately, while referring to the same implementation and
evidence. It is not a claim that a particular journal has reviewed the work
or that a production service has achieved the reported laboratory timings.

\begin{table}[H]
\centering\small
\begin{tabularx}{\linewidth}{L{0.21\linewidth}Y}
\toprule
Reader & Suggested route and decision to make \\
\midrule
Engineering leader & Start with the design origin, system composition and
walkthrough (Section~\ref{sec:walkthrough}); then read the results
(Section~\ref{sec:results}). Identify where preparation could overlap real work. \\
Security reviewer & Read the trust model (Section~\ref{sec:contract}),
preservation conditions (Section~\ref{sec:formal}) and limitations
(Section~\ref{sec:limits}). Identify which actor could bypass the boundary. \\
Research reader & Read related work (Section~\ref{sec:related}), the formal
argument and evaluation method (Section~\ref{sec:method}). Distinguish
assumptions, authored expectations and empirical observations. \\
Buyer or operator & Read the walkthrough, scale/deployment discussion and
evidence field guide (Appendix~\ref{app:field-guide}). Decide what must be
demonstrated in a workflow-specific pilot before enforcement. \\
\bottomrule
\end{tabularx}
\caption{Reading routes through one report, not separate claims for different audiences.}
\label{tab:reading-guide}
\end{table}

\paragraph{Evidence conventions.}
\emph{Measured} denotes execution-derived records from the described
configuration. \emph{Controlled} denotes a deliberately constructed input,
fault or comparison, not natural production traffic. \emph{Illustrative}
denotes an example used to explain a mechanism. \emph{Proposed} denotes an
extension whose performance or correctness has not been established by this
study. In particular, the scale-routing diagrams are proposed designs; they
must not be read as results from an implemented fleet scheduler. Transit
systems are cited as design inspiration, not as an endorsement or a transfer
of their reliability guarantees to this software.

\clearpage
\setcounter{tocdepth}{2}
\tableofcontents
\clearpage

\section{Introduction}
\label{sec:introduction}

An agent's permission to use a tool is not necessarily permission to execute
the particular request eventually sent to that tool. A storage write can
change after review; an approved deployment can acquire a different target;
an unused authorization can be copied to another worker. Conversely, repeating
an expensive review after the final request becomes ready may unnecessarily
delay work whose relevant facts have not changed. The design problem is to
separate reusable reasoning from facts that must still be checked at admission,
without quietly replacing current authorization with a cached allow bit.

\zerogate{} studies this separation for a deliberately narrow object: a
short-lived, exact-action, single-use pass. The issuer approves a concrete
canonical action and binds its approval to identity, authority, validity,
freshness evidence, and a consumption budget. At dispatch, a trusted adapter
reconstructs the final action, including the final payload commitment. A local
gate verifies the signed pass, checks the remaining conditions, and records
admission before the trusted caller sends the external request. Preparation is
useful only when the action is known early enough and the required evidence
remains valid. A changed payload requires another approval; unused preparation
is wasted work rather than a free optimization.

The design connects two existing project contracts.
\pcaa{} (Proof-Carrying Agent Actions) describes runtime governance
\citep{pcaa2026}; \cava{} (Canonical Action Verification and Attestation)
describes action representation and binding \citep{cava2026}.
This paper asks how an approval over a canonical action can
be revalidated and durably consumed at a later boundary. It does not reproduce
the whole \pcaa{} review workflow or claim a machine-checked derivation from
arbitrary policy. The current action representation includes a runtime field.
Consequently, a portable \emph{revalidation contract} does not imply identical
fingerprints across runtimes, and our experiment does not establish such
identity.

The name also reflects a limited systems analogy. Shiibashi describes Suica's
autonomous decentralized ticketing architecture, including local operation
and later information exchange, while Sony documents FeliCa as a contactless
IC-card system \citep{shiibashi2008suica,sonyfelica}. Preparing work before a
short interaction and separating that interaction from later reconciliation
is useful inspiration. We infer no undocumented FeliCa cryptographic,
transaction, timing, or revocation mechanism, and transfer no transit
performance figure to agent authorization.

Our claim is intentionally conditional. A locally persisted policy digest can
match a pass while both are stale. A receipt can survive a crash even though
no external effect occurred. Two independent regional nonce databases can
each consume the same copied pass. These are not implementation details that
can be hidden behind a claim of trust preservation; they delimit what the
construction can establish.

The contributions are:
\begin{enumerate}
  \item A precise separation of earlier exact-action approval, current
  revalidation, durable admission, and subsequent external execution, with a
  conditional decision-preservation proposition and explicit counterexamples.
  \item A reference adapter using Ed25519-signed passes and a SQLite
  write-ahead-log transaction for nonce, budget, and receipt consistency,
  together with controlled replay, quota-race, and crash protocols.
  \item A source-pinned semantic study and a real storage study that compare
  identical authorization work placed at different times, report failure
  denominators, and distinguish dispatch-boundary latency from complete
  lifecycle latency.
\end{enumerate}

This is an engineering and research contribution in contract specification,
integration, and measurement. Capabilities, attenuated credentials, local
policy evaluation, and durable duplicate suppression already supply much of
its foundation. The question is whether this particular composition is
correctly bounded and useful for the studied workload, not whether local
authorization or request binding was previously unknown.


\section{Design Origin}
\label{foundation-design-origin}

The design starts with a practical distinction: deciding that an action may
happen and checking that the action now presented is still the approved one
are different jobs. An engineer needs an enforceable interface between them;
an operator needs to understand why work proceeds or stops; a researcher
needs assumptions under which the separation preserves the decision.
\zerogate{} makes that interface explicit rather than treating a faster
permission check as the whole system. The transit comparison explains the
choice of architecture, while the later contract and experiments establish
its precise technical claims.

\subsection{Contributions of the transit sources}

Sony's FeliCa and JR East's Suica occupy different levels of this account.
FeliCa is a contactless IC-card technology; Suica is a deployed ticketing
system with cards, terminals, station servers, and a centre server. Sony's
\emph{The FeliCa System} describes reader/card communication, service-specific
access rights, and anti-tear transactions. In the documented interruption
case, when a card leaves the reader's range, uncommitted data are discarded
to restore the previous state \citep{sonyfelica}. These are concrete design
features of the card system, not evidence about an agent runtime. Their
relevance is the attention paid to both the successful interaction and the
state left behind when that interaction fails.

Sony also documents polling as card discovery: the reader/writer issues a
Polling command, and cards randomly choose among specified time slots to
reduce simultaneous-response collisions
\citep[Secs.~2.3.5 and~4.4.2]{foundation-sony-polling2026}. Its overview
credits time-slot collision detection and avoidance with fewer transaction
steps \citep{sonyfelica}. The design lesson is to distinguish discovery and
contention handling from authorized data access, and to remove avoidable
interaction steps rather than checks. Radio-response collisions are not
replay or resource conflicts in software; those require their own mechanisms
at the action boundary.

Shiibashi's account of Suica supplies the deployment perspective
\citep{shiibashi2008suica}. Gates process cards and retain data without
accessing the centre server for each interaction. Stored data move through
station servers to the centre when communications are available, where
usage information can be consolidated and matched. The account also
describes coexistence with magnetic tickets and the need to preserve basic
operation despite equipment failure. Local operation is therefore one part
of a coordinated system, not a claim that central administration or network
communication has disappeared. This distinction matters equally when an
enterprise places an authorization check near its tools while retaining
central policy ownership and audit responsibilities.

Suica's interaction design offers another useful lesson. Field testing found
that users presented cards differently; the touch-and-go gesture was
introduced to make the interaction more dependable
\citep{shiibashi2008suica}. The corresponding agent-design question is what
must be stable when the runtime presents an action. A predictable submission
boundary, an immutable outgoing payload, and an explicit response to missing
evidence are more useful than an interface that invites repeated guesses.
This is an engineering interpretation of the published account, not a
Japanese design standard or a claim of Sony or JR East endorsement.

A complementary control-engineering influence comes from Shinkansen
ride-control systems. Yagishita's JR East presentation describes car-body tilting for
curves and full active suspension to suppress lateral body motion
\citep[p.~19]{foundation-jreastshinkansen2013}. Yamada's air-spring study
distinguishes sensor-based and map-based tilting control for rolling stock
\citep{foundation-airspringtilt2017}. This motivates adapting support to
observable conditions rather than treating every request identically.
Section~\ref{scale-support} develops the software interpretation in terms of
expiry, justified freshness, contention, and review deadlines. There is no
physical equivalence between train dynamics and agent actions, and no
timing-performance or safety guarantee transfers from these sources.

\begin{figure}[tbp]
\centering
\begin{tikzpicture}[
  x=1cm,y=1cm,
  foundationbox/.style={zg node,
    text width=4.35cm,minimum height=1.25cm,inner sep=6pt,
    align=center,font=\small},
  foundationflow/.style={zg flow},
  foundationmap/.style={zg flow,dashed,draw=zgLine},
  foundationhead/.style={font=\small\bfseries,align=center}]
  \node[foundationhead,anchor=west] at (0,2.5)
    {TRANSIT: published system features};
  \node[font=\small] at (2.45,1.92) {Prepare};
  \node[font=\small] at (7.65,1.92) {Local interaction};
  \node[font=\small] at (12.85,1.92) {Later coordination};
  \node[foundationbox,fill=zgMint] (foundationcard) at (2.45,0.98)
    {Card issuance and\\service-specific data};
  \node[foundationbox,fill=zgMint] (foundationtransit) at (7.65,0.98)
    {Reader/card exchange;\\local ticket processing};
  \node[foundationbox] (foundationcentre) at (12.85,0.98)
    {Station and centre\\data transfer / matching};
  \draw[foundationflow] (foundationcard) -- (foundationtransit);
  \draw[foundationflow] (foundationtransit) -- (foundationcentre);
  \node[foundationbox,fill=zgBlue]
    (foundationpass) at (2.45,-1.1)
    {Review exact action;\\mint signed \actionpass{}};
  \node[foundationbox,fill=zgCream]
    (foundationgate) at (7.65,-1.1)
    {Revalidate; durably record\\nonce, quota, and receipt};
  \node[foundationbox,fill=zgMint]
    (foundationeffect) at (12.85,-1.1)
    {Attempt external effect;\\verify and reconcile};
  \draw[foundationmap] (foundationcard) -- (foundationpass);
  \draw[foundationmap] (foundationtransit) -- (foundationgate);
  \draw[foundationmap] (foundationcentre) -- (foundationeffect);
  \draw[foundationflow] (foundationpass) -- (foundationgate);
  \draw[foundationflow] (foundationgate) -- (foundationeffect);
  \node[foundationhead,anchor=west,text=zgInk] at (0,-2.08)
    {AGENT TRANSLATION: distinct approval, admission, and effect};
  \node[font=\small,align=center,text width=15.3cm] at (7.65,-2.82)
    {Failure stays explicit: incomplete evidence means no dispatch;\\
     a committed admission does not prove an external effect.};
\end{tikzpicture}
\caption{Transit-to-agent design translation. Dashed links indicate analogy,
not protocol equivalence. The upper row combines documented FeliCa and Suica
features; the lower row describes the reference authorization lifecycle.
No time scale or transferable safety guarantee is implied.}
\label{foundation-transit-translation}
\end{figure}

\subsection{From inspiration to engineering obligations}

Figure~\ref{foundation-transit-translation} translates these observations
into three obligations. First, prepare work whose inputs are already known.
Policy review and signing can precede readiness when the final action is
available, but payload-dependent approval cannot precede knowledge of that
payload. The tested pass is a one-use approval for an exact action, not a
transit balance or a reusable permission to perform a class of operations.
If preparation becomes invalid or the action is abandoned, its cost remains.
This motivates reporting preparation and complete lifecycle alongside the
shorter admission boundary.

Second, make the boundary small without making its evidence optional.
The gate must compare the observed action with the signed commitment, check
the supported authority and freshness bindings, and establish that the
nonce and applicable budget can be consumed. Locality does not justify
omitting a dependency. A resource fact that only a remote system knows may
require a fresh query or a separately justified consistency protocol.
Matching locally stored digests establishes agreement with that snapshot,
not automatic knowledge of the current external world. Where the required
facts cannot be established, the useful response is fresh review or denial,
not an optimistic interpretation of a cached approval.

Third, define interruption at the actual transaction boundary. FeliCa's
documented anti-tear behavior motivates asking which state changes are
indivisible, but does not supply the answer for distributed software.
Here, one local database transaction couples nonce consumption, applicable
quota usage, and the admission receipt. External execution follows that
transaction and is not rolled back with it. A crash can therefore leave a
valid admission record and no external effect. Recovery needs the retained
request and destination-specific retry semantics, rather than a second
consumption of the same pass. The analogy is useful precisely because it
forces this boundary to be named instead of suggesting that all failures
have become atomic.

For an operator, the resulting experience should be understandable rather
than merely quick: preparation identifies the exact proposed work, admission
explains whether its evidence still suffices, and outcome records distinguish
authorization from completion. A changed target should create a new decision,
not a silently expanded approval. Neither Sony's transaction timings nor
Suica's service volumes are used as agent performance targets. The value
borrowed from transit is the separation of responsibilities; its usefulness
for agents must be demonstrated on agent workloads.

\section{System Composition}
\label{foundation-system-composition}

The roles can be read without adopting the project vocabulary first:
\cava{} describes the concrete action, \pcaa{} frames who may authorize it,
and \actionpass{} carries the signed approval. \zerogate{} is the software
gate that checks the approval against the outgoing request and records
one-time consumption before permitting dispatch. This is an interface map;
it does not imply that every component was exercised together in the study.

\zerogate{} fits between a decision about authority and an attempt to cause
an effect. Its relationship to other OSuite research components is best
expressed through interfaces, not a claim that their entire pipelines were
executed together. \cava{} addresses action representation and attestation
\citep{cava2026}; \pcaa{} frames deployer-owned authority and the surrounding
governance process \citep{pcaa2026}. Bounded Action Firewall (BAF) contributes
the architectural vocabulary of constrained execution, while Agent Runtime
Exposure Graph (AREG) describes how action, authority, destination, and
evidence records can be joined \citep{foundation-baf2026,foundation-areg2026}.
Figure~\ref{foundation-composition} separates these relationships from the
reference path measured in this report.

\begin{figure}[tbp]
\centering
\begin{tikzpicture}[
  x=1cm,y=1cm,
  foundationbox/.style={zg node,
    text width=4.3cm,minimum height=1.55cm,inner sep=6pt,
    align=center,font=\small},
  foundationinterface/.style={foundationbox,dashed,fill=zgBlue},
  foundationflow/.style={zg flow},
  foundationlink/.style={foundationflow,dashed,draw=zgLine},
  foundationlabel/.style={font=\small,align=center,fill=white,inner sep=2pt}]
  \node[foundationinterface] (foundationskill) at (2.45,2.9)
    {\textbf{Neutral skill boundary}\\Requested scope;\\never a grant};
  \node[foundationinterface] (foundationpcaa) at (7.65,2.9)
    {\textbf{\pcaa{} authority}\\Decision owner;\\review obligations};
  \node[foundationinterface] (foundationbaf) at (12.85,2.9)
    {\textbf{BAF execution bounds}\\Scope, expiry,\\reuse constraints};
  \node[foundationbox,fill=zgBlue] (foundationcava) at (2.45,0.55)
    {\textbf{Adapter + \cava{}}\\Concrete action\\and fingerprint};
  \node[foundationbox,fill=zgBlue] (foundationissuer) at (7.65,0.55)
    {\textbf{Reference issuer}\\Fixture policy;\\signed exact pass};
  \node[foundationbox,fill=zgMint] (foundationdispatch) at (12.85,0.55)
    {\textbf{Trusted dispatch}\\External effect;\\then verification};
  \node[foundationbox,fill=zgBlue]
    (foundationobserve) at (2.45,-1.9)
    {\textbf{Final observation}\\Rebuild the\\outgoing action};
  \node[foundationbox,fill=zgCream]
    (foundationconsume) at (7.65,-1.9)
    {\textbf{\zerogate{} + local store}\\Revalidate, consume;\\persist receipt};
  \node[foundationinterface,fill=zgGray] (foundationareg) at (12.85,-1.9)
    {\textbf{AREG exposure}\\Coverage, decisions,\\and effect links};
  \draw[foundationlink] (foundationskill) -- (foundationpcaa);
  \draw[foundationlink] (foundationpcaa) -- (foundationissuer);
  \draw[foundationlink] (foundationbaf.south) -- ++(0,-0.35)
    -| (foundationissuer.north east);
  \draw[foundationflow] (foundationcava) -- (foundationissuer);
  \draw[foundationflow] (foundationcava) -- (foundationobserve);
  \draw[foundationflow] (foundationissuer) --
    node[foundationlabel,right=2pt] {signed pass} (foundationconsume);
  \draw[foundationflow] (foundationobserve) -- (foundationconsume);
  \draw[foundationflow] (foundationconsume.north east) -- (foundationdispatch.south west);
  \draw[foundationlink] (foundationconsume) -- (foundationareg);
  \draw[foundationlink] (foundationdispatch) -- (foundationareg);
  \node[font=\small,align=center,text width=15.3cm] at (7.65,-3.55)
    {Solid: reference responsibilities and data path.
     Dashed: architectural interfaces.\\
     Review or deny releases no dispatch; the database transaction ends locally.};
\end{tikzpicture}
\caption{Composition without conflating interface design and integration
evidence. Boxes are logical responsibilities, not separate hosts. BAF and
AREG are not tested integrations here, and the fixture issuer does not
execute the complete \pcaa{} pipeline.}
\label{foundation-composition}
\end{figure}

\subsection{Requested, granted, and residual authority}
\label{foundation-authority-layers}

A vendor-neutral skill boundary describes what work a skill asks to perform:
operations, resources, destinations, relevant systems, and constraints. It
does not confer authority merely by being well formed or bundled with the
skill. The runtime authorization layer evaluates the request under deployer
policy and the task's identity and context; it may narrow or deny the request.
\zerogate{} then evaluates the resulting grant; only a successfully committed
execute admission consumes it. These are interface
responsibilities, not a claim that collaborators have frozen a common
neutral-skill schema.

Let $R$ be the set of concrete actions described by the requested scope,
$G$ the set granted by the authorizer, and $E_t$ the residual set eligible
under the gate's supported evidence and consumption state at admission time
$t$. The intended relationship is
\begin{equation}
  E_t \subseteq G \subseteq R.
  \label{foundation-narrowing}
\end{equation}
The inclusion $G \subseteq R$ is an upstream authorizer obligation. The
reference gate checks the observed action against the signed grant, not
against the original skill request; \texttt{authority.requested\_ref} is
\texttt{null} in the cloud passes.
An empty grant denies the request; a narrower grant excludes actions that
the skill requested but policy did not allow. Resource and destination sets
narrow by removing members. Numeric limits narrow by tightening intervals:
a lower maximum or higher minimum excludes possibilities. Equality
constraints, such as an exact state version, require equality rather than
numerical proximity. The current evaluator distinguishes numeric values from
numeric strings and treats missing required evidence as insufficient, not
as an unrestricted scope.

For the tested pass $p$, exact-action binding imposes the additional
restriction $E_t \subseteq G \cap \{a_p\}$, where $a_p$ is the exact action
bound into $p$. After successful consumption, that pass has no residual eligible
action. A later action can lie inside the requested business scope and still
require a new pass. In particular, changing a payload to something apparently
less risky does not preserve its fingerprint. These sets describe the
intended narrowing contract; the gate's eligibility remains relative to its
evidence, not a guarantee against unobserved external change. Prospective
bounded-class approval would need an explicit membership rule, delegation
policy, and shared accounting. It is not established by this one-use design.

\subsection{Concrete interface contracts}

A practical integration should keep four records distinguishable: the
skill's request, the issuer's grant, the gate's admission receipt, and the
destination's effect evidence. Their references connect the lifecycle,
but one record cannot substitute for another. The request reference
explains why authority was sought; the pass identifies what was approved;
the receipt records the decision and, for execute admissions, consumption;
and the effect record describes what happened externally. Version identifiers and documented field types
make these handoffs inspectable. An adapter must also specify how it
obtains observations and prevents bypass, since an accurate record written
after an uncontrolled effect is not pre-execution enforcement. These
obligations are useful acceptance criteria even before every component
shares a transport or storage implementation.

\paragraph{Representation: \cava{}.}
The adapter supplies a concrete action object; normalization produces the
versioned representation used for hashing. Relevant fields include runtime,
operation, target, systems touched, and metadata carrying payload commitments.
At admission, the trusted adapter reconstructs that object from the outgoing
request rather than echoing the approved fingerprint. Normalization is not a
semantic oracle: it cannot discover every hidden shell expansion, omitted
side effect, or misleading adapter assertion. Unsupported interpretation
must remain unsupported. The present representation includes runtime and
other adapter-related fields, so consistent serialization within one
representation does not establish equal hashes across runtimes. Mapping
quality and authorization correctness remain separate questions.

\paragraph{Authority: \pcaa{} and the issuer.}
The conceptual upstream output identifies the decision owner, granted scope,
review outcome, and evidence obligations. A permitted concrete action can
then be bound into a signed pass with holder/session identifiers, validity,
policy/state references, nonce, and budget. \pcaa{} explains where final
authority belongs; it does not make a skill's request authoritative. In the
experiment, a deterministic fixture issuer performs the review and signing,
not the complete \pcaa{} route--review--proof workflow. The issuer signature
authenticates its statement, not a machine-checked derivation of policy
correctness. Holder equality likewise depends on trusted adapter observations
rather than independently authenticating the presenter.

\paragraph{Execution bounds: BAF and \zerogate{}.}
BAF's architectural concerns translate into explicit scope, expiry, reuse,
and evidence conditions, rather than a second implicit permission layer.
\zerogate{} evaluates one exact grant, returning execute, review, or deny.
A committed execute admission consumes the nonce and applicable quota.
In the durable adapter, execute is released only after the local consumption
transaction commits; time validity is checked using the post-lock admission
sample, not promised at eventual external execution. A composition should
assign one owner to nonce and quota accounting rather than maintain
independent stores that can each admit a copied pass. These contracts explain
compatibility with BAF, not a measured BAF integration.

\paragraph{Exposure and evidence: AREG.}
An exposure interface can join runtime/adapter identity, action fingerprint,
authority reference, destination, admission decision, receipt, and later
effect evidence. It should also declare whether an adapter enforces before
execution, relies on an interruptible hook, or only observes. An observe-only
record can support investigation but cannot establish that a denied action
was prevented. AREG provides an architectural destination for such records;
the reported experiments do not test a live AREG ingestion path. In
particular, an admission receipt and an effect acknowledgement must remain
distinct facts even when a dashboard displays them together.

\subsection{Illustrative cross-runtime translations}

Consider a skill requesting a write of one report to an internal object
store. A tool-call adapter might receive a structured operation and object
identifier; an SDK adapter might observe a client method and byte buffer.
Both should resolve the actual destination and bind the final bytes and
length before requesting approval. This is a contract illustration, not
another cloud experiment. Under the current runtime-inclusive representation,
the two projections need not share a fingerprint or pass. Each must obtain
approval for its own concrete representation. A changed object name or body
requires a new exact-action decision even if both requests still look like
``write the report.''

For a second illustrative case, a refund skill requests at most 100 units
for a specified order; policy grants at most 60. A browser adapter and an
API adapter would each need to establish the order, amount, currency, and
destination, not merely recognize a refund button or endpoint. Preparing an
exact pass for 40 does not authorize a later request for 30 under that same
pass: the authority bound may be respected while the exact-action binding
has changed. Missing currency or an unobserved browser submission prevents
establishing the necessary contract. Neither example is an empirical result
or a claim that these adapters already provide complete mediation.

The composition is valuable because it localizes responsibility. Skill
authors state needs, authorizers decide scope, adapters expose concrete
requests, the gate records consumption, and evidence systems connect that
decision to what followed. This gives engineers contracts to implement,
researchers assumptions to challenge, and commercial readers a way to ask
which boundary a proposed deployment actually controls.

\section{Relationship to Prior Work}
\label{sec:related}

\paragraph{Capabilities and contextual attenuation.}
An action pass is capability-like: it conveys restricted authority that a
verifier can check without repeating the issuer's reasoning. Complete
mediation still requires checking authority for the actual operation;
Saltzer and Schroeder explicitly warn that cached access decisions need a
reliable account of changes \citep{rwSaltzer1975}. Macaroons provide
decentralized delegation through contextual caveats, including third-party
discharges \citep{rwMacaroons2014}. Biscuit combines signed authorization
tokens with attenuating logic, and its documentation explicitly describes
per-request attenuation and body-hash binding
\citep{rwBiscuitIntro,rwBiscuitRequest}. Exact-request restriction is therefore
not itself novel. \zerogate{} selects a narrower operating point: a trusted
issuer mints a concrete single-use approval, the adapter reconstructs the
request, and a local transaction couples consumption with an admission
receipt. Unlike a broadly delegated reusable credential, this construction
pays minting and persistence costs for each distinct approved action. It does
not implement a general decentralized attenuation language.

\paragraph{Proof-carrying authorization and local policy engines.}
Appel and Felten's \emph{Proof-Carrying Authentication} provides a foundational
account of presenting evidence that a small checker can validate
\citep{rwAppel1999}. Our signature authenticates the issuer's statement; it is
not a proof that arbitrary issuer code correctly implemented policy. The
proposition below makes issuer soundness and dependency completeness explicit
assumptions. Nor is avoiding a remote policy call unique to passes. OPA
documents indexing, partial evaluation, and WebAssembly execution
\citep{rwOPAPerformance,rwOPAWasm}; Cedar is designed for expressive, fast,
safe, and analyzable authorization \citep{rwCedar2024}. An embedded engine
with locally available facts can evaluate at the boundary without contacting
a remote reviewer. The relevant comparison is among locations and
representations of policy work, evidence acquisition, and state coordination,
not a presumption that policy engines require remote execution.
Our embedded fixture baseline is not an OPA or Cedar benchmark and cannot
establish superiority over either.

\paragraph{Leases, revocation, and freshness.}
A pass expiry resembles a bound on delegated authority, but is not by itself
the cache-consistency protocol of Gray and Cheriton's leases
\citep{rwLeases1989}. A lease involves a promise about which changes may occur
while cached information is used. A timestamp does not create that promise.
Biscuit's revocation identifiers likewise require a verifier to learn which
identifiers have been revoked \citep{rwBiscuitRevocation}. For \zerogate{},
an epoch comparison is meaningful only relative to the source of that epoch.
An isolated verifier cannot distinguish two otherwise identical histories
that differ only by an unseen remote revocation. It must accept possible
staleness, reject while disconnected, or rely on a separately justified lease
or communication bound. Offline immediate revocation is not a claim here.

\paragraph{Durability, duplicate suppression, and external effects.}
Linearizability locates an operation at a point between invocation and
response \citep{rwHerlihy1990}; we use that abstraction for admission to one
consumption database, not for the entire cloud workflow. RIFL couples
duplicate detection and completion information to recoverable operations
\citep{rwRIFL2015}. Helland distinguishes local transactions from interactions
between independently managed entities \citep{rwHelland2007}. Transactional
outboxes address a corresponding dual-write problem by persisting delivery
intent with local state, but normally still require retry and receiver-side
idempotence \citep{rwAWSOutbox}. Stripe documents a concrete idempotency-key
contract with parameter checks and retention rules, rather than an
unqualified eternal exactly-once promise \citep{rwStripeIdempotency}.
Our receipt is an admission record, not automatically a durable delivery
queue. Its atomicity with nonce consumption prevents a second local
admission; it cannot atomically commit an Azure operation.

Taken together, these precedents support a defensible positioning:
\zerogate{} specializes established authorization and durability techniques
around a canonical agent-action boundary. Its distinctive object of study is
the explicit link between an earlier exact-action decision, the dependencies
that justify preserving it, a single local consumption domain, and measured
dispatch versus lifecycle cost. The scientific value is in exposing those
conditions and testing the composition. We make neither a first-ever claim
nor a claim that this composition dominates existing authorization systems.

\section{Contract and Trust Model}
\label{sec:contract}

\subsection{Objects, actors, and phases}

The \emph{issuer} evaluates a policy over a proposed final action. The
\emph{adapter} obtains that pass and observes the actual action at the
boundary. The \emph{gate} checks the pass and updates a local consumption
store. The \emph{effect service} accepts the dispatched request and talks to
the storage system. In the reference study, the adapter and gate run in the
trusted client, while the review and effect endpoints share a service.
Figure~\ref{fig:architecture} describes logical responsibilities, not
independent failure domains or a requirement that every box has its own host.

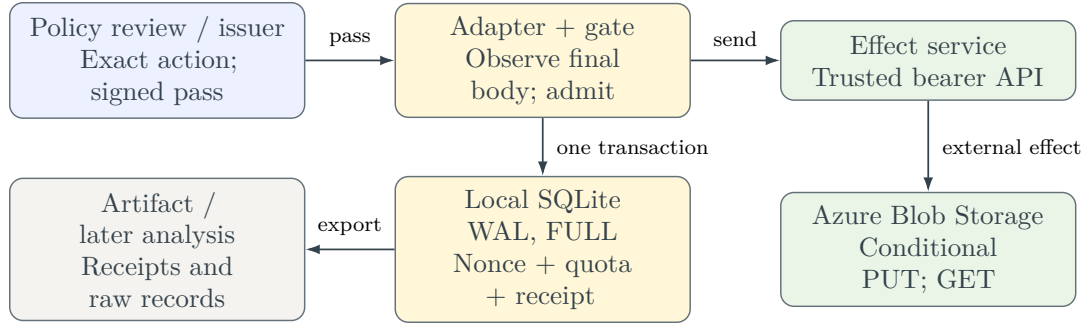
\begin{figure}[tbp]
\centering
\begin{tikzpicture}[
  box/.style={zg node,align=center,font=\small,
    text width=3.55cm,minimum height=1.05cm,inner sep=5pt},
  flow/.style={zg flow},
  lab/.style={font=\scriptsize,align=center,fill=white,inner sep=2pt}]
  \node[box,fill=zgBlue] (issuer) at (0,0) {Policy review / issuer\\Exact action; signed pass};
  \node[box,fill=zgCream] (gate) at (5.1,0) {Adapter + gate\\Observe final body; admit};
  \node[box,fill=zgMint] (service) at (10.2,0) {Effect service\\Trusted bearer API};
  \node[box] (evidence) at (0,-2.5) {Artifact / later analysis\\Receipts and raw records};
  \node[box,fill=zgCream] (db) at (5.1,-2.5) {Local SQLite WAL, FULL\\Nonce + quota + receipt};
  \node[box,fill=zgMint] (blob) at (10.2,-2.5) {Azure Blob Storage\\Conditional PUT; GET};
  \draw[flow] (issuer) -- node[lab,above=3pt] {pass} (gate);
  \draw[flow] (gate) -- node[lab,above=3pt] {send} (service);
  \draw[flow] (gate) -- node[lab,right=3pt] {one transaction} (db);
  \draw[flow] (service) -- node[lab,right=3pt] {external effect} (blob);
  \draw[flow] (db) -- node[lab,above=3pt] {export} (evidence);
\end{tikzpicture}
\caption{Logical architecture. The external request is outside the admission
transaction. The effect service checks a caller-supplied receipt binding,
not an independently authenticated proof of gate persistence.}
\label{fig:architecture}
\end{figure}

A pass contains the canonical action and its fingerprint; issuer and key
reference; pass identifier and nonce; holder, workspace, and session
bindings; permitted operations, resources, destinations, systems, and
constraints; validity bounds; policy and state digests and revocation epoch;
and budget and receipt requirements. The general core also represents
additional predicates and governance conditions. The durable research adapter
supports a restricted subset rather than pretending to acquire arbitrary
external evidence.

For the cloud workload, the action is a particular Azure Blob write with a
specific object identity, payload SHA-256 commitment, and byte length.
Each pass has a unit budget and is not reusable across different writes.
The service's deterministic fixture policy checks its fixed holder and
workspace, allowed destination and operation, and payload bounds. This is an
executable research issuer, not a human review process or an implementation
of every policy envisioned by \pcaa{}.

There are three distinct times. At \emph{preparation}, the issuer approves and
signs. At \emph{admission}, the gate revalidates and durably consumes. At
\emph{execution}, the service attempts the external effect. A success at one
time does not assert success at the next. The core uses
\Execute{}, \Review{}, and \Deny{} outcomes: only \Execute{} permits dispatch;
\Review{} requests fresh evidence or another decision; \Deny{} rejects the
submitted attempt. The binary semantic experiment tests dispatch permission,
not whether every non-executing case was routed to the correct non-executing
lane.

\subsection{Observation and identity boundaries}

The final observation must be derived from what will actually be sent, not
from the approved fingerprint echoed by the caller. The cloud client
recomputes the payload hash and byte length from the final body, reconstructs
the canonical action, and obtains a new fingerprint. The adapter must then
keep that action and payload immutable through dispatch. Neither a signature
nor a canonical representation closes a later time-of-check/time-of-use gap
if another component can replace the outgoing bytes.

Holder and session fields are assertions supplied by the trusted adapter.
Their equality checks do not authenticate the presenter. The service uses
a shared bearer credential for privileged endpoints; the issuer's Ed25519
signature authenticates the pass issuer, not possession of a holder's private
key. In particular, this is not the sender-constrained proof-of-possession
mechanism specified by DPoP \citep{rwDPoP2023}. HTTPS and an operator-trusted
endpoint or CA are required when deploying the client across an untrusted
network. Fetching a public key from an unauthenticated or substituted metadata
endpoint would not establish issuer trust.

The effect service checks that the supplied receipt says \Execute{} and
matches the submitted pass hash and action fingerprint. It does not obtain
the local gate database, independently attest its transaction, or require a
gate-signed receipt. Its audit explicitly labels this evidence
\texttt{caller\_supplied\_binding\_only}. A malicious bearer-credential holder
could synthesize those fields. Thus the API is a trusted-runner research
boundary, not a secure public admission endpoint for arbitrary untrusted
actors. Complete mediation in this architecture is an obligation of the
trusted runner and its deployment, not a property proved by receipt syntax.

\subsection{Freshness and consumption domains}

The gate owns a persisted snapshot containing a policy digest, a state digest,
and a revocation epoch. It reads this snapshot inside the admission
transaction, rather than accepting replacements in each caller observation.
This prevents a caller from simply asserting the desired epoch. It does not
make the database current. The service and gate snapshots can diverge, and
there is no implemented automatic coherence or revocation distribution
protocol between them.

Moreover, the service's initial state digest hashes backend configuration,
not the contents or current authorization state of Azure. A successful
comparison therefore attests consistency with that configured snapshot,
not a fresh observation of arbitrary resources. Policies requiring other
state predicates or a separately bound revocation digest are not supported
by pretending that caller-supplied values are authoritative: the adapter
does not import those values as fresh evidence.

The \emph{consumption domain} $\Omega$ is one persistent gate database with
coordinated access to its nonce and quota tables. Processes using that same
database can contend on one admission record. Separate databases, including
independently initialized regional copies, constitute separate domains.
SQLite WAL itself requires cooperating processes on the same host and allows
one writer at a time \citep{sqliteWAL}. No global replay property follows from
the current same-host implementation.

\section{Conditional Decision Preservation}
\label{sec:formal}

This section states a specification obligation, not an unconditional theorem
about an arbitrary signed object. In particular, the synchronous decision is
defined independently of the fast gate; defining the oracle to be the gate
would make agreement circular.

In ordinary terms, the claim is that an earlier approval may still justify
release if the relevant action and facts have not changed and the remaining
checks succeed. The conditions below state who must make that true. The
claim runs in one direction: a successful gate admission must be justified;
it does not require the gate to admit every action a fresh reviewer might
allow. A valid signature alone cannot supply these conditions.

\begin{definition}[Decision and dependency projection]
Let $a$ be the immutable action to be dispatched, including its final payload
commitment. Let $C_v(a)$ be its canonical representation under version $v$ and
\[
  h_v(a)=\operatorname{SHA256}(\SJSON(C_v(a))).
\]
$\SJSON$ denotes the implementation's recursively sorted-key JSON
serialization, not a claim of RFC~8785/JCS conformance. Let $w$ be a coherent
authoritative world state, including policy and evaluator versions, identity,
authority, resource facts, and any required evidence or revocation state.
Let $s$ be local consumption state and $\kappa$ a logical admission identity.
An independently specified synchronous decision is
\[
 D_{\rm sync}(a,\kappa,w,s)\in\{\Execute,\Review,\Deny\}.
\]
For an eligible policy fragment with dependency description $d$, require
predicates $K,V$ and a complete projection $\pi_d$ such that
\begin{equation}
 \begin{split}
 &D_{\rm sync}(a,\kappa,w,s)=\Execute\\
 &\qquad\Longleftrightarrow\
 K(C_v(a),\pi_d(w))\ \land\ V(a,\kappa,w,s).
 \end{split}
 \label{eq:factorization}
\end{equation}
$K$ captures the approval fact established at minting; $V$ captures remaining
admission conditions, including current identity constraints, validity,
unusedness, available budget, and receipt obligations.
\end{definition}

Factorization~\eqref{eq:factorization} must be justified for the actual policy.
It is not established merely by putting a policy digest into a pass.
A dependency can be encoded by a value, version, or commitment only if its
interpretation and update discipline make the required equality meaningful.
Relevant changes not represented in $\pi_d$ invalidate the proposed
factorization. Negative decisions need not admit the same decomposition.

Consider a pass $P$ minted over $a_m$ and dependency value $c_m$. Let
$\operatorname{Admit}_{\Omega}(P,a,\kappa,\lambda)$ denote a successful durable
admission assessed at time $\lambda$ in $\Omega$, and let
$s_{\lambda}^{-}$ be the state before its consumption. In the reference gate,
$\lambda$ is the post-lock validity-check sample within the transaction;
it is not the later commit-visibility point or return time. Consumption is
serialized by the transaction, while the authorization claim is explicitly
about this assessment instant. The admission key in the
implementation is the tuple
\[
  k(P)=(\text{issuer id},\text{workspace id},\text{nonce}).
\]
Pass-identifier usage and pass-hash consistency are checked separately.
The identifier $\kappa$ aligns the same logical operation in the specification
and implementation; it is not an additional proof of authority.

\paragraph{Assumptions.}
The proposed guarantee requires all of the following.
\begin{enumerate}[label=\textbf{A\arabic*.}]
  \item \textbf{Sound approval and authentic binding.}
  The trusted issuer signs only after establishing
  $K(C_v(a_m),c_m)$. The signed hash binds all fields used by the gate;
  verification, key selection, serialization, and collision resistance are
  sound. A valid signature alone does not establish the issuer's policy
  correctness.
  \item \textbf{Complete dependencies.}
  Equation~\eqref{eq:factorization} holds for the admitted policy fragment,
  including all facts whose change could invalidate the earlier approval.
  The verifier and canonicalization versions have agreed meanings.
  \item \textbf{Faithful observations and mediation.}
  The trusted adapter supplies authentic identity and session context,
  reconstructs the final action, and dispatches exactly that action only
  after successful admission. Equal checked representations denote the
  same policy-relevant operation in this adapter.
  \item \textbf{Current coherent evidence.}
  At $\lambda$, the checked dependency value is authoritative and satisfies
  $\pi_d(w_\lambda)=c_m$. This requires an independently justified freshness
  mechanism or stability promise; an unchanged stale local row is
  insufficient.
  \item \textbf{Sound residual checks.}
  Acceptance of the gate's residual checks implies
  $V(a,\kappa,w_\lambda,s_\lambda^{-})$. Clocks and observation sources are
  trusted, missing evidence does not produce \Execute{}, and applicability
  limits are enforced.
  \item \textbf{Atomic durable admission.}
  In $\Omega$, checking unusedness and budget, consuming the key and
  applicable quota, and recording the bound receipt form one linearizable
  durable transaction. Successful admission is released only after commit.
  Recovery preserves committed state; the host does not roll back or edit
  that state outside the protocol.
\end{enumerate}

\begin{proposition}[One-sided conditional preservation]
\label{prop:preservation}
Under A1--A6, for an eligible pass and action,
\[
 \operatorname{Admit}_{\Omega}(P,a,\kappa,\lambda)
 \ \Longrightarrow\
 D_{\rm sync}(a,\kappa,w_\lambda,s_\lambda^{-})=\Execute.
\]
Within $\Omega$, there is at most one successful admission for $k(P)$, and
that admission has a durably bound receipt before trusted dispatch.
\end{proposition}

\begin{proof}
Authentic binding and sound issuance establish $K(C_v(a_m),c_m)$.
The action check and faithful adapter identify the policy-relevant action
with $a_m$, and current dependency equality transfers this approval fact to
$K(C_v(a),\pi_d(w_\lambda))$. Sound residual checking establishes
$V(a,\kappa,w_\lambda,s_\lambda^{-})$. The independently justified
factorization then gives the claimed synchronous \Execute{} decision.
Atomic serialization orders competing admissions: after the first commit,
the key is consumed, so no later transaction can satisfy unusedness.
The same transaction persists the corresponding receipt, and commit precedes
release to the trusted dispatcher. These last properties concern admission,
not the outcome of the external call.
\end{proof}

\paragraph{Scope of agreement.}
This is a safety implication, not a converse. A synchronous policy may allow
an action whose pass is absent, expired, unsupported, or unavailable because
the local store failed. The gate may conservatively refuse it. Nor does the
proposition identify \Review{} with \Deny{} or prove correct escalation.
On a subset where preparation is sound and all remaining requirements hold,
both paths authorize; outside it, no universal three-way decision equivalence
is claimed.

The comparison uses $s_\lambda^{-}$, not state after consumption. Asking the
synchronous policy whether the nonce is still unused after the fast path
consumed it would compare different operations. Likewise, a quota reservation
must be attributable to this admission rather than interpreted as newly
available capacity for a second action.

\paragraph{Snapshot-relative implementation claim.}
The durable adapter can enforce consistency with its local snapshot and
consumption database. It does not supply A4 for arbitrary current-world
facts. Its unconditional implementation-level claim must therefore be
weaker: admission satisfies the implemented checks relative to the persisted
snapshot, subject to the trusted adapter and storage assumptions. Lifting
that result to current synchronous authorization requires the missing
freshness premise. Our semantic fixtures explicitly supply observations;
they do not experimentally establish an external coherence protocol.

\paragraph{Counterexamples and temporal limits.}
If a remote authority revokes access while the gate retains the old epoch,
the pass and database can agree while a current synchronous policy denies.
If recipient balance or deployment freeze status affects policy but is
omitted from the dependency projection, the theorem's factorization is false.
If the adapter hashes one body and sends another, action matching is
irrelevant to the actual effect. If a trusted API bearer fabricates a
receipt, the service-side binding check does not establish admission.
If two regions use separate nonce stores, both can admit a copy.
Each example violates a named premise rather than constituting evidence for
an unqualified preservation claim.

Finally, the proposition is about the admission instant. A permission or
resource condition can change after admission and before execution. A policy
requiring validity at the external effect needs an appropriate lease,
receiver-side enforcement, transactional resource operation, or another
check. The core's time interval is inclusive:
$\texttt{not\_before}\leq t\leq\texttt{expires\_at}$.
The implementation does not imply a strict upper bound or a global clock
agreement merely by checking that interval.

\section{Reference Implementation}
\label{sec:implementation}

\subsection{Signed exact-action approval}

The new implementation resides in \path{scripts/zerogate-revalidation/},
using the core evaluator in
\path{packages/cava-core/src/action-pass.js}.
\path{services.mjs} implements review, effect, metadata, snapshot, and audit
endpoints; \path{client.mjs} supplies the authenticated HTTP/HTTPS client.
Review requests are idempotent by operation identifier and request hash.
Repeating the same request can retrieve its original signed response;
reusing the identifier for a different review request is rejected. This
deduplication does not turn one exact-action pass into authority for multiple
actions.

\path{durable-gate.mjs} implements pass signing, signature verification, and
durable admission. The serializer recursively sorts object keys and
preserves array order, using JSON representations for supported scalar
values. The durable path rejects non-finite numbers and unsupported or lossy
non-JSON objects. These conventions define the artifact's serialization
contract; they are not a demonstrated cross-language canonical-JSON standard.

The pass hash is SHA-256 over the serialized payload excluding its own
\texttt{pass\_hash} field. The Ed25519 signature is over the domain-separated
string \texttt{osuite.zerogate.pass.v1}, a newline, and that hash.
The verifier recomputes the hash, selects the configured trusted key, checks
its binding to the pass's issuer key reference, and verifies the detached
signature and encoding. Ed25519's signature scheme is standardized separately
\citep{rfc8032}; these implementation checks do not imply a new cryptographic
construction. Issuer private keys are persisted, and the service refuses to
silently replace missing key material in an existing state directory.

\subsection{The admission transaction}

The gate requires a persistent database. It configures and verifies
\texttt{journal\_mode=WAL} and \texttt{synchronous=FULL}, validates schema
versioning, enables foreign keys, and uses an immediate write transaction.
SQLite documents the relevant WAL and synchronization behavior
\citep{sqliteWAL,sqliteSynchronous}. Durability still assumes the operating
system, filesystem, and device honor the required persistence operations;
process-kill tests are not a substitute for power-loss testing.

Within one \texttt{BEGIN IMMEDIATE} transaction, the gate:
\begin{enumerate}
  \item Reads the persisted snapshot and freezes the submitted pass into a
  supported JSON representation before verification and evaluation.
  \item Checks the signature, issuer/workspace scope, nonce, pass-identifier
  consistency, positive integer units, per-pass capacity, and any configured
  aggregate quota.
  \item Replaces caller-supplied replay, quota, and snapshot assertions with
  database values, then calls the core action-pass evaluator.
  \item Constructs a receipt binding the decision to the pass, observed
  action, issuer, workspace, nonce, units, and snapshot. On \Execute{}, it
  updates pass usage and applicable aggregate consumption and inserts the
  unique consumed nonce.
  \item Inserts the receipt and commits before returning a successful
  admission to the caller.
\end{enumerate}

The nonce key is unique within issuer and workspace. A separate pass-usage
key rejects reusing a pass identifier with changed contents. A deferred
foreign-key constraint links a consumed nonce to its receipt. Review and
deny decisions can also be durably recorded when the store is functioning,
but they do not consume the nonce as an execution. Receipt-availability
flags supplied to the core here are transaction obligations discharged by
the adapter, not independently observed proof supplied by the caller.

On transaction failure, including failure at commit, the adapter releases no
\Execute{} result. It returns a storage-failure denial without claiming a
persisted receipt. An ambiguous commit can nevertheless have consumed the
nonce; fail-closed output sacrifices availability rather than authorizing a
second attempt. A later recovery procedure must consult durable state.
Reopening an existing database does not overwrite its snapshot or quota
state with initialization arguments.

The default validity clock is sampled after acquiring the write lock, not
before waiting for it. Here the admission instant denotes that in-transaction
check sample, not the subsequent commit return or external effect. A
real-process regression holds the writer lock until a signed pass expires;
the waiting gate then records a denial without consuming authority. This
regression exposed a pre-lock sampling defect during revalidation. We fixed
it and repeated the complete cloud and fault campaigns; the measurements
reported below use the corrected, source-pinned implementation.

\subsection{Effect execution and crash semantics}

The cloud effect service uses managed identity to access Azure Blob Storage;
this storage credential is distinct from the bearer credential protecting
the research API. It attempts a conditional block-blob PUT with
\texttt{If-None-Match: *}, then reads back and checks content. Azure documents
Put Blob and its conditional-header semantics
\citep{azurePutBlob,azureConditional}. Stable operation identifiers, content
checks, and conditional object creation make retries of this narrow write
operation inspectable. An existing object is not accepted merely because
its name matches.

The service maintains its own effect records in a different SQLite database.
There is no transaction spanning the gate database, the service database,
and Azure. A process can fail after an Azure write but before recording it
locally. A later retry may reconcile the existing object by reading and
checking it. This is a receiver-specific retry strategy, not a general
exactly-once guarantee for payments, mail, or arbitrary APIs.

If the runner crashes after admission but before dispatch, the nonce and
receipt can survive while the effect is absent. Retrying admission is then
rejected as replay. Recovering and dispatching the already admitted operation
requires its retained intent and an explicit recovery policy. The fault
campaign demonstrates such a controlled recovery using saved request
material; the gate does not implement an autonomous durable outbox worker.
The service rechecks current pass validity on effect requests, including
retries. Demonstrated recovery completes within the original five-minute
window; repeating the same review request returns the stored pass rather
than renewing it. Post-expiry recovery therefore requires an additional
renewal and reconciliation protocol not evaluated here.
Even with an outbox, a lost acknowledgement leaves uncertainty unless the
receiver's idempotency or transactional contract resolves it.

\section{A Governed Write, from Intent to Evidence}
\label{sec:walkthrough}

The distinction between approval, admission and execution becomes easier to
inspect when each refers to an actual artifact. This section follows
\ReportCaseId{}, the first indexed prepared request in the first
concurrency-eight trial. The selection rule is descriptive, not a choice of
the fastest observation. The request writes 128 bytes to the experiment's
private Azure container. Its contents are synthetic test data, but the
write, signed pass, local receipt and downloaded object are real records
from the controlled cloud deployment. No customer document or customer
workload is implied.

Table~\ref{tab:integration-map} identifies the components behind this
example. They are logical responsibilities, not four independently secured
services. Sections~\ref{sec:contract} and~\ref{sec:implementation} describe
the trust assumptions and executable implementation respectively.

\begin{table}[!htbp]
\centering\small
\begin{tabularx}{\linewidth}{L{0.22\linewidth}Y}
\toprule
Role & Reference component and responsibility \\
\midrule
Approval issuer & The review endpoint in \path{services.mjs} checks the
fixture policy and signs the pass using service-side issuer key material. \\
Observing client & The trusted runner, using \path{client.mjs}, obtains
the pass, reconstructs the final body and action, and dispatches after
admission. It holds the research API bearer credential. \\
Durable gate & \path{durable-gate.mjs} verifies the pass and commits
consumption and receipt state in local SQLite. It calls the core
\path{action-pass.js} evaluator; that evaluator alone is not a durable
store or a non-bypassable dispatch boundary. \\
Effect service & The effect endpoint in \path{services.mjs} uses its
Azure managed identity to write the object. That storage identity is
distinct from the caller's API credential and the issuer's signing key. \\
\bottomrule
\end{tabularx}
\caption{Integration map for the controlled cloud example. Credential
separation by purpose is not proof that an authorized bearer holder cannot
bypass the reference gate.}
\label{tab:integration-map}
\end{table}

\subsection{Fix the object before asking for authority}

The request identifies an HTTP runtime, the executable
\texttt{AzureBlob.PutBlob}, the operation \texttt{put\_blob}, the destination
container, the object identifier, the payload length and a SHA-256 commitment
to the final body. The destination describes where the effect will occur;
the payload commitment describes what is being written. These are different
obligations. Checking the container alone would still permit a different
document to be written to an approved place.

The adapter obtains holder and session context from its trusted runtime.
Those fields are compared later, but their presence does not independently
authenticate the caller. In this implementation, that distinction matters:
the research API is protected by a bearer credential, and the runner that
holds it is within the trusted boundary. A production design that gives the
acting agent unrestricted access to that credential has not created an
independent gate simply by copying these fields into its requests.

An application can prepare this particular action only after these
policy-relevant details are fixed. It may prepare other context earlier,
but that is not the same as minting the exact pass tested here. If a later
agent step changes the content, this is a new approval problem, even when
the tool name and destination remain unchanged. Request identifiers are
useful for correlation and deduplication; they are not substitutes for the
content commitment.

\subsection{Review the final proposal and carry the signed result}

The issuer checks the experiment's deterministic policy over the request.
The returned pass, \ReportCasePassId{}, permits one write of that exact
canonical action. Its authority scope includes the operation, object,
container, storage system and a maximum payload-size constraint. The cloud
case's \texttt{requested\_ref} is null: this request is not an execution of
the separate vendor-neutral skill-contract experiment.
When a skill request is present, the upstream authorizer must keep its
grant within that request, as specified in
Section~\ref{foundation-authority-layers}. This gate checks the signed grant
and exact action, not the original skill contract. Even a smaller payload
requires a new pass if its action fingerprint changes.

The action fingerprint commits to the canonical action. The pass hash
commits to the larger authorization object, including that action binding,
holder, policy/state references, validity, nonce and budget. The detached
Ed25519 envelope authenticates the issuer's statement about that pass hash.
These objects have different jobs. Hash equality alone does not identify an
issuer, and an issuer signature does not prove that its decision procedure
was sound. Their composition gives the gate a specific statement to verify
against a configured trust root.

Figure~\ref{fig:walkthrough-bindings} shows the evidence path across the
local admission and external storage boundaries. The complete
hashes remain in the artifact; selected values are reproduced in
Appendix~\ref{app:field-guide} so a reader can locate this observation.

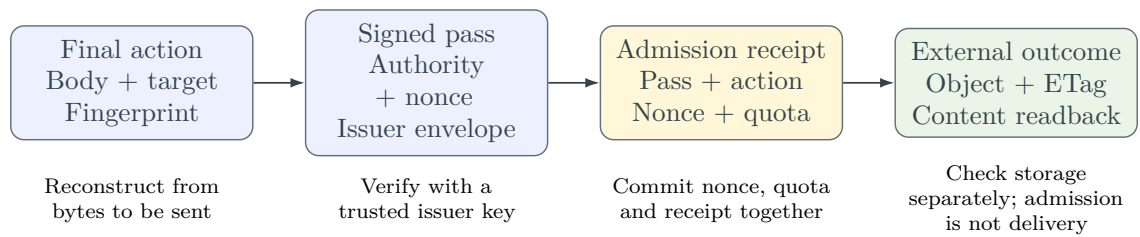
\begin{figure}[!htbp]
\centering
\begin{tikzpicture}[
  b/.style={zg node,align=center,text width=2.85cm,
    minimum height=1.35cm,font=\small,inner sep=5pt},
  f/.style={zg flow},
  note/.style={font=\scriptsize,align=center,text width=3cm}]
  \node[b,fill=zgBlue] (a) at (0,0) {Final action\\Body + target\\Fingerprint};
  \node[b,fill=zgBlue] (p) at (3.9,0) {Signed pass\\Authority + nonce\\Issuer envelope};
  \node[b,fill=zgCream] (r) at (7.8,0) {Admission receipt\\Pass + action\\Nonce + quota};
  \node[b,fill=zgMint] (e) at (11.7,0) {External outcome\\Object + ETag\\Content readback};
  \draw[f] (a) -- (p);
  \draw[f] (p) -- (r);
  \draw[f] (r) -- (e);
  \node[note] at (0,-1.55) {Reconstruct from bytes\nobreakspace to be sent};
  \node[note] at (3.9,-1.55) {Verify with a\\trusted issuer key};
  \node[note] at (7.8,-1.55) {Commit nonce, quota and\nobreakspace receipt together};
  \node[note] at (11.7,-1.55) {Check storage separately;\nobreakspace admission is not delivery};
\end{tikzpicture}
\caption{Binding chain for the governed write. The first three boxes lead
to a local admission decision; the external outcome is established through
a separate service and storage path. Arrows express data dependencies, not
measured time or a distributed atomic transaction.}
\label{fig:walkthrough-bindings}
\end{figure}

\subsection{Reconstruct, revalidate and durably admit}

Before dispatch, the client hashes the final body again and reconstructs
the action rather than trusting the pass's fingerprint as an observation.
The gate verifies the issuer envelope and checks the holder, local snapshot,
validity, required authority and unused nonce. Its transaction reads local
consumption state, not replay or quota assertions supplied by the caller.
This makes a repeated presentation a database consistency question rather
than a race between two agents that both claim to be first.

In this case the result is \Execute{}. Receipt 1761 binds the action
fingerprint and pass hash to the issuer, workspace, nonce, requested unit,
snapshot and decision. The nonce is consumed and the receipt is committed
before the result is returned. Here the requested unit is one budget unit,
not a description of the skill's requested authority. The receipt's validity-check timestamp is an
admission observation; it must not be used as a timestamp for the Azure
write. Process and service clocks are not subtracted to construct latency.

The final check is still substantial work. It reconstructs evidence,
verifies a signature, checks local state and writes durable records. The
fast-path design does not replace these tasks with a cache hit on an
\texttt{approved=true} flag. Its opportunity is to have already completed
the issuer round trip before the dispatch worker receives the action.

\subsection{Execute, then establish what actually happened}

After successful admission, the client submits the same request material to
the effect service. The service checks the pass and the receipt's declared
bindings, then uses its managed identity to access Azure. Conditional object
creation prevents an ordinary retry from blindly overwriting an existing
object. A subsequent read compares stored content with the expected
commitment. For this observation, the returned ETag is \ReportCaseETag{}.
The separately downloaded object also matches its recorded payload hash.

That chain allows a reader to ask several different questions: did the
issuer sign this pass; does it bind this action; is there a corresponding
local admission record; and do the retained storage bytes match the approved
payload? It does not let an unauthenticated receipt prove its own durability,
nor does it prove who physically sent an API request against a malicious
runner. The executable artifact verifier and the trust assumptions answer
different parts of that inquiry.

The experiment's effect endpoint labels receipt evidence as
\texttt{caller\_supplied\_binding\_only}. A receiver requiring independently
authenticated admission would need an additional mechanism, such as a
trusted gate assertion with a receiver-enforced validation contract and
protected signing authority. The present service does not supply that
mechanism. An external verifier can enrich the evidence or participate in
approval, but post-execution enrichment cannot retroactively enforce a
missing pre-execution gate.

\subsection{Read the two clocks, not just the fast one}

The \emph{boundary} starts when a dispatch worker picks up this action and
ends when the gate returns; worker pickup is not security admission.
The \emph{lifecycle} starts at this action's preparation and ends at its
explicit service verification. Queueing before preparation and the later
direct Azure downloads are outside that lifecycle. Both modes perform the
same issuance and gate work; synchronous mode performs issuance after
worker pickup. Section~\ref{sec:timing-definitions} gives the exact timer
definitions and exclusions, and Figure~\ref{fig:timing} shows their order.

Table~\ref{tab:walkthrough-times} derives its values directly from this
record's monotonic-clock durations. The unassigned interval is computed as
lifecycle minus the four displayed timed phases; it includes prepared-batch
dwell and otherwise unassigned harness intervals, not a separately
instrumented queue measurement. It is not inferred by subtracting wall-clock
timestamps from different hosts.

\begin{table}[!htbp]
\centering\small
\ReportCaseTimingTable
\caption{One selected prepared action, not a percentile or representative
service guarantee. The short dispatch boundary coexists with a much longer
recorded lifecycle. All durations belong to the selected record identified
at the start of this section.}
\label{tab:walkthrough-times}
\end{table}

In an interactive deployment, the relevant product question would be when
the user actually requested execution relative to preparation. The
experiment does not measure that user clock. If preparation and its wait
were already complete while the user was reviewing a finished export, the
last step might benefit. If the user was waiting for that preparation from
the outset, the same work would be visible delay. This difference is why a
short final boundary is useful evidence but is not yet a measured
``tap-and-go'' user-experience result.

\subsection{Handle interruption without inventing a second authority}

The separate fault campaign kills a process after durable admission and
before dispatch. The receipt and nonce remain, and the service has no
effect record. Its HTTP 404 checks that record, not Azure directly; the
controlled pre-dispatch stop supplies the additional protocol context.
A repeated admission is rejected; otherwise an ambiguous outcome
could become permission to spend the same authority twice. The controller
instead recovers the retained original request and delivers that already
admitted operation. The storage service's idempotency/content contract is
essential to this particular recovery. Delivery completes within the
original five-minute pass validity window. The service checks validity
again on effect requests, including retries, so a retained admission does
not authorize indefinite recovery. Repeating the same review request returns
the stored pass, not a renewed one. Post-expiry recovery needs an explicit
renewal and reconciliation design that this experiment does not establish.

Figure~\ref{fig:walkthrough-recovery} distinguishes consuming authority from
retrying delivery. It also explains why a durable receipt must not simply
be labeled a completed action in a user interface. A useful operations view
would show \emph{prepared}, \emph{admitted}, \emph{outcome pending} and
\emph{effect verified} separately. These are an interpretation of the
recorded protocol states, not a claim that this study deployed new screens
in OSuite Studio, the product's operator interface, or an autonomous
recovery service.

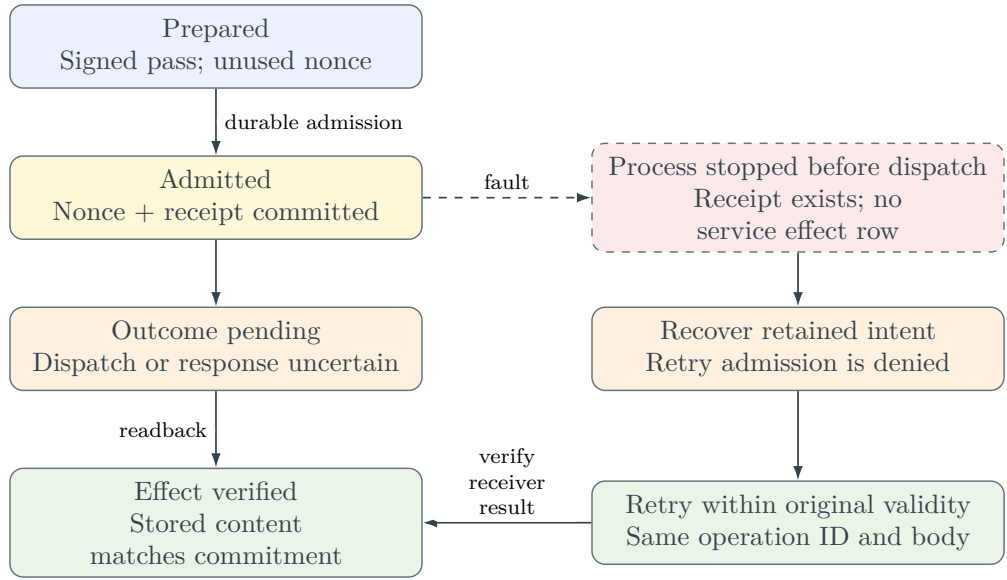
\begin{figure}[!htbp]
\centering
\begin{tikzpicture}[
  state/.style={zg node,text width=5.1cm,
    minimum height=1.1cm,align=center,font=\small,inner sep=5pt},
  f/.style={zg flow},
  lab/.style={font=\scriptsize,fill=white,inner sep=3pt}]
  \node[state,fill=zgBlue] (signed) at (0,0) {Prepared\\Signed pass; unused nonce};
  \node[state,fill=zgCream] (admitted) at (0,-2) {Admitted\\Nonce + receipt committed};
  \node[state,fill=zgPeach] (pending) at (0,-4) {Outcome pending\\Dispatch or response uncertain};
  \node[state,fill=zgMint] (verified) at (0,-6.3) {Effect verified\\Stored content matches commitment};
  \node[state,fill=zgRose,dashed] (stopped) at (7.7,-2) {Process stopped before dispatch\\Receipt exists; no service effect row};
  \node[state,fill=zgPeach] (recover) at (7.7,-4) {Recover retained intent\\Retry admission is denied};
  \node[state,fill=zgMint] (deliver) at (7.7,-6.3) {Retry within original validity\\Same operation ID and body};
  \draw[f] (signed) -- node[lab,right] {durable admission} (admitted);
  \draw[f] (admitted) -- (pending);
  \draw[f] (pending) -- node[lab,left] {readback} (verified);
  \draw[f,dashed] (admitted) -- node[lab,above] {fault} (stopped);
  \draw[f] (stopped) -- (recover);
  \draw[f] (recover) -- (deliver);
  \draw[f] (deliver) -- node[lab,above,align=center,text width=1.9cm]
    {verify\\receiver\\result} (verified);
\end{tikzpicture}
\caption{Admission and recovery state view. The dashed branch corresponds
to the controlled post-commit/pre-dispatch process kill. Recovery is performed
by the test controller using retained intent, not by a shipped autonomous
outbox. Re-delivery is not a second successful admission, and the shown
recovery does not extend pass validity.}
\label{fig:walkthrough-recovery}
\end{figure}

The missing-evidence and changed-payload tests exercise different branches:
they stop before execution rather than invoke recovery. In the local
policy-refresh case, updating the snapshot makes the old pass return review;
a separately signed replacement with a new pass ID and nonce then obtains
execute. That case does not dispatch a cloud effect. Keeping these branches visible is
important commercially: an operator needs to distinguish ``waiting for
authority,'' ``already admitted but delivery uncertain,'' and ``rejected,''
because each requires a different intervention.

\section{Evaluation Method}
\label{sec:method}

The evaluation keeps three evidence classes separate. The semantic study
checks decisions on authored fixtures. The cloud study measures actual
review, admission, dispatch, and storage verification through one deployed
configuration. The fault study checks controlled histories, persistent
records, and effects. None can substitute for the others: semantic receipt
flags are not disk writes, cloud timings are not policy ground truth, and a
successful crash scenario does not prove all failure histories.

\subsection{Cloud design and comparison}

\path{cloud-study.mjs} defines a target design of \PlanCloudTrials{} trials,
\PlanCloudPerCell{} actions per trial and mode, \PlanCloudModes{} modes, and
\PlanCloudLevels{} concurrency levels (\PlanCloudConcurrency{}), for
\PlanCloudAttempts{} logical action attempts. This is a planned sample
count, not a count of completed writes. The generated result table reports
actual attempts, successes, and failures. The operation is a real Azure Blob
PUT followed by GET-based verification; a logical action can cause several
HTTP requests, including additional readbacks and retries. Logical actions
must not be reported as a count of Azure HTTP operations.

The cloud deployment uses the D4s\_v3 class in Central US, zone 1, with
four virtual CPUs and 16\,GB of memory per VM. The client reports an Intel
Xeon 8272CL and the service an Intel Xeon 8370C; a shared VM class is not an
identical physical processor. The recorded cloud runtime is Node.js 24.21.0
with SQLite 3.53.4. Both VMs use Ubuntu 24.04 LTS, image version
24.04.202609040. The WAL is on a 64\,GB StandardSSD\_LRS operating-system
disk, not Premium SSD. HTTP(S) clients use keep-alive connections with a
30-second socket timeout; this is not a hard end-to-end deadline. The gate's
SQLite busy timeout is 5,000\,ms. These environment facts describe the
configured deployment, not independently sampled cloud environments.
Semantic timings come from a separate Node.js 22.18.0, macOS/arm64
environment. Its manifest does not retain CPU, RAM, or OS-release details;
those missing collection-time facts cannot be recovered from the cloud
configuration. Semantic and cloud timings must not be pooled.

Both cloud modes call the same deterministic issuer, obtain an exact-action
signed pass, reconstruct the final body, call the same durable gate, and use
the same effect and verification endpoints. In \emph{synchronous} mode,
worker admission is recorded before remote review and signing. In \emph{prepared}
mode, the entire cell's batch is reviewed first using the configured worker
pool, and the dispatch phase begins only after that preparation phase
finishes. The modes move minting; they do not eliminate it, replace it with a
cheaper decision, or amortize one pass over different actions.

Preparation is charged to each action's lifecycle, including the wait after
its own preparation while the remainder of the batch finishes and workers
become available. Thus prepared execution can have a shorter ready-to-dispatch
interval and a \emph{longer complete lifecycle}. We report both.
The batch design evaluates placement
of known work, not successful prediction of unknown future agent requests.
It does not include wasted preparations for actions never selected by an
agent, so deployments with speculation must additionally count those costs.

Requests use fresh operation identifiers and payload sizes cycling through
128, 1,024, and 3,072 bytes: respectively 54, 53, and 53 actions in each
160-action cell. The client uses a bounded-concurrency closed-loop pool, not
an open-loop arrival process. Concurrency levels run in order 1, 8, then 32;
within each level the five trials use synchronous-first order in trials
1, 3, and 5, and prepared-first order in trials 2 and 4. The same client
connection pools and gate database remain open across cells, and no cloud
warm-up observations are discarded. The study is therefore not a
fully randomized crossover, and requests are not identical paired operation
identifiers across modes. Review deduplication cannot be used to skip minting
in one mode by reusing another mode's identifiers. Ordinary cloud passes
exercise per-pass single use; aggregate quota contention is examined
separately in the fault campaign. Separating review into a prior batch changes
the offered request pattern and service contention. The comparison therefore
does not isolate a cryptographic improvement under identical arrivals.

Cloud concurrency denotes asynchronous workers in one client process.
Each local \texttt{gate.commit} call is synchronous, so these calls serialize
within that process while remote requests can overlap. The cloud study is
therefore not a measurement of 32 parallel gate writers. Separate-process
contention on a shared gate is tested by the fault campaign instead.

\subsection{Timing definitions}
\label{sec:timing-definitions}

All elapsed times are measured using the client's monotonic clock.
Pass validity uses host wall-clock time instead. The artifact does not
retain a measured client--service clock offset, so it cannot quantify
cross-host validity error or a clock-skew tolerance.
Let $t_p$ start preparation, $t_m$ mark receipt of the signed pass, $t_r$
mark admission to a dispatch worker, $t_g$ mark return from successful
durable admission, $t_e$
mark return from effect dispatch, and $t_v$ mark completion of explicit
verification. The worker-admission-to-dispatch boundary and lifecycle are
\[
 T_{\rm boundary}=t_g-t_r,\qquad
 T_{\rm lifecycle}=t_v-t_p.
\]
In synchronous mode, $t_r$ precedes preparation; in prepared mode, $t_r$
follows preparation and batch dwell. In both modes the worker has already
acquired the item from its pool before $t_r$ is taken. The raw field
\texttt{ready\_to\_dispatch} therefore measures this internal boundary, not
latency from exogenous request arrival, time waiting to acquire a worker,
or a queue-inclusive service-level objective. The separate gate timer starts before
final runtime-observation reconstruction and includes that work together
with gate evaluation and persistence. It is not a bare SQL or signature-only
microbenchmark. Effect latency includes the service's work, potentially
including its internal readback; verification latency measures the subsequent
explicit GET request.

\begin{figure}[tbp]
\centering
\begin{tikzpicture}[
  seg/.style={zg node,align=center,font=\scriptsize,minimum height=0.75cm,
    inner sep=2pt},
  zgmark/.style={zg flow,line width=0.45pt},
  txt/.style={font=\scriptsize,align=center}]
  \node[anchor=west,font=\small\bfseries] at (0,1.35) {Synchronous};
  \node[seg,fill=zgBlue,minimum width=3.28cm] at (1.7,0) {Review + mint};
  \node[seg,fill=zgCream,minimum width=1.98cm,text width=1.7cm] at (4.45,0) {Observe + gate};
  \node[seg,fill=zgMint,minimum width=3.28cm] at (7.2,0) {Effect request};
  \node[seg,minimum width=3.28cm] at (10.6,0) {Verify};
  \draw[zgmark] (0,0.7) -- node[txt,above] {ready-to-dispatch} (5.5,0.7);
  \draw[zgmark] (0,-0.7) -- node[txt,below] {lifecycle: preparation through verification} (12.3,-0.7);
  \node[anchor=west,font=\small\bfseries] at (0,-1.6) {Prepared};
  \node[seg,fill=zgBlue,minimum width=2.33cm] at (1.225,-2.5) {Review + mint};
  \node[seg,fill=zgPeach,minimum width=2.33cm] at (3.675,-2.5) {Batch dwell};
  \node[seg,fill=zgCream,minimum width=2.33cm] at (6.125,-2.5) {Observe + gate};
  \node[seg,fill=zgMint,minimum width=2.33cm] at (8.575,-2.5) {Effect request};
  \node[seg,minimum width=2.38cm] at (11.05,-2.5) {Verify};
  \draw[zgmark] (4.9,-1.8) -- node[txt,above] {ready-to-dispatch} (7.35,-1.8);
  \draw[zgmark] (0,-3.2) -- node[txt,below] {lifecycle: includes preparation and batch dwell} (12.3,-3.2);
\end{tikzpicture}
\caption{Logical ordering only; segment widths are unscaled and encode no
measured latency or speedup. Readiness starts before review in synchronous
mode and after preparation/dwell in prepared mode.}
\label{fig:timing}
\end{figure}
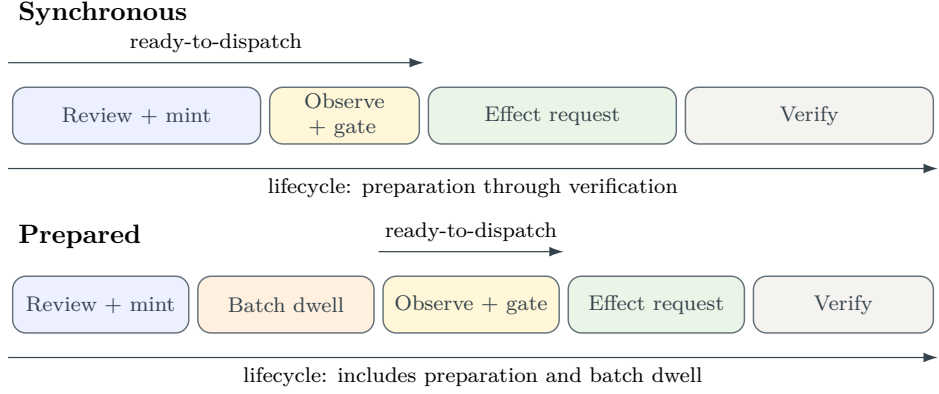

The lifecycle begins when that action starts preparation, not when the
experiment first enqueues the whole workload. It includes subsequent batch
and worker waiting but excludes queueing before its own preparation starts.
Small harness intervals are included by elapsed-clock measurement rather
than reconstructed by summing rounded phase statistics. Percentiles from
different phases are not additive.

Every attempted action remains in the JSONL record stream, with its error and
available phase timings. Cloud latency tables use successful attempts only;
they must be read alongside attempted, successful, and failed counts.
Preparation, admission, effect, or verification failure is not filled with a
zero latency or silently treated as a successful fast rejection. A trial
without successes contributes failures but cannot contribute a successful
latency mean. A separate verifier can summarize all finite observed phases,
including partially timed failures; those diagnostics must not be mislabeled
as the success-conditioned table.

\paragraph{Direct storage retrieval.}
After the timed cloud run, a separate author-operated Azure CLI path lists
and downloads the actual blobs without using the experimental HTTP service.
\path{verify-readback.mjs} compares downloaded bytes with each recorded
request body, recomputes SHA-256, and compares the Azure inventory's size and
ETag with the recorded effect. The evidence is
\path{azure-direct-readback/}, \path{azure-direct-inventory.json}, and
\path{azure-direct-verification.json} under the artifact root. These checks
are outside the timed lifecycle. They corroborate stored content through
another retrieval path, not independent human replication, a signed Azure
attestation, or proof of the local gate's fsync.

\subsection{Semantic corpus and methods}

\path{semantic-benchmark.mjs} renders \PlanSemanticScenarios{} authored
scenarios into \PlanSemanticDomains{} domains: payment, mail, deployment,
and storage. The Cartesian product gives \PlanSemanticUnique{} unique
scenario--domain combinations. \PlanSemanticRepeats{} complete repeats with
varied request/session identifiers yield a target of
\PlanSemanticInputs{} input instances per method. Each input is evaluated
by \PlanSemanticMethods{} methods. Repeating a template changes observations,
not the number of independently designed semantic cases.

Construction starts with a core-built pass and a matching action/holder
fixture; a scenario then mutates the relevant fields. Five of the 73
templates expect allow and 68 expect not allow. Across four domains and
five repeats this yields 100 expected-allow and 1,360 expected-not-allow
inputs per method. This shared construction is another reason to distinguish
implementation conformance from independent validation of the specification.

The scenarios cover valid actions, narrowed authority, budget boundaries,
action and payload substitutions, falsely repeated old fingerprints, holder
and session mismatch, policy/state/revocation drift, missing predicates,
authority widening and type changes, exhaustion, replay, receipt and privacy
requirements, slot conflict, and compound mutations. The domain renderings
are hand-constructed canonical objects, not independent parsers of production
requests. No real payment, deployment, or mail effect is executed by this
semantic study.

Each scenario has a readable authored expectation for whether dispatch is
allowed. The expected label is metadata and is not supplied as an input to
the decision methods. Nevertheless, labels, fixtures, and the reference
policy come from the same authors. They are test expectations, not
independently collected ground truth. The complete embedded fixture policy
uses a separate implementation of the checks rather than calling the core
validator or hash helpers, which helps detect implementation disagreement
but does not create an independent policy specification.

\begin{table}[tbp]
\centering
\small
\begin{tabularx}{\linewidth}{L{0.32\linewidth}Y}
\toprule
Method identifier & What the method actually does \\
\midrule
\path{action_pass_core} &
Core evaluator including receipt construction; fixture flags assert receipt
availability, not measured persistence. \\
\path{embedded_policy_v1} &
Separately implemented complete authored fixture policy, decision only. \\
\path{without_action_binding_v1} &
Embedded policy without the action-binding acceptance check. \\
\path{without_freshness_v1} &
Embedded policy without policy/state/revocation freshness and predicate
checks; revocation-list checks remain. \\
\path{without_receipt_v1} &
Embedded policy without receipt-evidence acceptance checks. \\
\path{cava_action_binding_only_v1} &
Canonical-action equality and hash binding only, not the complete
\cava{} protocol. \\
\path{raw_scope_v1} &
Operation and resource membership only, not a representative capability
system or production IAM policy. \\
\bottomrule
\end{tabularx}
\caption{Semantic comparison methods. Omission variants diagnose this fixture
policy; they are not measurements of OPA, Cedar, Macaroons, or Biscuit.}
\label{tab:methods}
\end{table}

Input order uses seed 20260921 and is shuffled within complete scenario--domain blocks;
method order is shuffled per input. Timing starts after fixture creation
and cloning, measures an actual synchronous method call, and excludes JSONL
serialization and disk output. It includes cold/JIT effects because the run
does not discard a dedicated warm-up phase. The core constructs receipts,
while decision-only methods do not. Algorithms and hashing implementations
also differ. Their timing comparison is descriptive, not an isolated causal
estimate of the cost of a single check or a claim about production throughput.
This benchmark does not measure Ed25519 verification, network review, SQLite
persistence, or cross-runtime action identity.

\subsection{Outcomes, denominators, and uncertainty}

For semantic reporting, \Execute{} maps to allow, while both \Deny{}
and \Review{} map to not allow. Exceptions and invalid outcomes map to a
separate error column. Let $N_-$ and $N_+$ be all evaluations with
expected not-allow and allow labels, respectively, including error outcomes.
The conditional disagreement fractions are
\[
 \widehat p_{\rm FA} =
 \frac{\#(\text{expected not allow, actual allow})}{N_-},
 \qquad
 \widehat p_{\rm FNA} =
 \frac{\#(\text{expected allow, actual not allow})}{N_+}.
\]
Errors are retained in those denominators but are neither credited as correct
rejections nor silently added to the non-error numerator. The confusion
matrix and error count must accompany these fractions. If overall failure
on expected-allow inputs is desired, its numerator is the sum of actual
not-allow and error outcomes, explicitly labeled as such. The binary label
design cannot support a claim of accurate escalation between deny and review.

\path{statistics.mjs} uses nearest-rank latency percentiles and reports
conditional fractions with Wilson intervals. Mean-latency intervals resample
complete trial clusters using a percentile bootstrap with 1,000 resamples
and seed 20260921. These intervals
describe repeated trials of this experiment. The semantic corpus has
strongly related templates and repeated same-process observations, so a
Wilson interval's independent-Bernoulli interpretation is only a descriptive
approximation, not a population security guarantee. Cloud trials on the same
deployment likewise do not represent independent cloud deployments. Empty
samples or insufficient clusters produce unavailable estimates rather than
invented precision. No observed absence of a fixture error proves absence of
bugs.

For the core's observed zero false allows in 1,360 expected-not-allow
inputs, the descriptive 95\% Wilson interval is [0, 0.282\%]. For zero
false not-allows in 100 expected-allow inputs it is [0, 3.699\%]. The smaller
allow denominator gives a wider interval. Neither interval turns the
correlated authored corpus into an independent sample of real-world attacks
or establishes a deployment error-rate bound.

\subsection{Controlled fault campaign}

\path{fault-study.mjs} exercises fresh success and negative submissions for
payload mutation, expiry, missing authority, and invalid signatures. It
also tests explicit local policy-snapshot refresh, same-pass races,
distinct-pass aggregate-quota races, a crash after durable commit, and a
discarded acknowledgement followed by retry. The table reports each case's
observed outcome rather than treating the existence of a test as a pass.

Each of the ten campaign scenarios contributes one retained history, not a
repeated estimate of failure frequency. Each race starts eight separate
worker processes sharing one gate database.
The same-pass case checks that only the first committed admission consumes
the nonce. The quota case uses distinct signed passes contending on a
configured three-unit aggregate limit: three admissions commit and five
are denied. No external effect is dispatched in that quota-only scenario;
the measurement is admission conservation, not eight competing Azure writes.
Policy-refresh and quota-race phases use freshly generated local
issuer keys; they are not measurements of remote issuer refresh propagation
or Azure quota management.

The crash worker reports that its transaction and saved result have
completed, then is killed before effect dispatch. The controller checks
receipt recovery, rejection of readmission, and absence of a service effect
record before explicitly recovering the retained operation. HTTP 404 means
that no service row exists; it is not a direct Azure absence check and alone
cannot exclude a write made before a service-record failure. The controlled
pre-dispatch stop supplies the narrower recovery context. The acknowledgement case
intentionally discards a response already received by the client and retries
the same request. It is not an injected packet-loss experiment. Expiry is
tested by advancing the trusted observation time, not by sleeping until
wall-clock expiry. Service GET evidence includes a hash recomputed from
persisted bytes; the client does not receive the raw object body through that
evidence endpoint.

Unit tests additionally exercise rollback, deferred-constraint failure,
schema checks, and persistent state across reopen. These are useful checks
of implementation obligations, but the campaign is neither exhaustive fault
model checking nor a power-loss, corrupted-storage, or Byzantine-host study.

\section{Results and Interpretation}
\label{sec:results}

All measured values below are supplied by the result include generated from
the frozen artifact, rather than transcribed into the manuscript. The cloud
artifact contains \CloudRecords{} recorded logical attempts,
\CloudSuccesses{} successes, and \CloudFailures{} failures.
The semantic artifact contains \SemanticCases{} inputs per method,
\SemanticEvaluations{} total method evaluations, and
\SemanticUniqueCases{} unique scenario--domain combinations.
Counts distinguish planned workload size from recorded and successful
observations.

\begin{table}[tbp]
\centering
\footnotesize
\CloudResultsTable
\caption{Cloud results by mode and concurrency; times are milliseconds.
Successes are shown over attempted actions, retaining failures in the
denominator. Boundary and lifecycle latency summaries condition on successful
attempts. Preparation and batch dwell are charged to lifecycle.}
\label{tab:cloud}
\end{table}

Table~\ref{tab:cloud} shows lower prepared worker-admission-to-dispatch p95
at each concurrency level. However, prepared mean complete lifecycle is
longer at every level, as shown in Figure~\ref{fig:lifecycle}; lifecycle p95
is also higher. Earlier review and the batch's actual dwell are paid rather
than removed. The larger boundary separation at higher concurrency includes
closed-loop service contention and different offered arrival patterns, not
a pure cryptographic gain. This experiment demonstrates a cost-placement
tradeoff, not net end-to-end acceleration.

For readers comparing deployment choices, Table~\ref{tab:paired-tradeoff}
puts the two clocks together. Its percentage is the arithmetic reduction
of the observed boundary p95, not a throughput gain, reduction in total
authorization work, or comparison with an earlier Studio deployment.
The synchronous mode is a controlled placement alternative using the same
research implementation, not a benchmark of a competing product.

The separate direct Azure retrieval validates all \CloudSuccesses{} successful
objects: downloaded bodies, SHA-256 values, lengths, and ETags agree with the
recorded requests and effects. This is stronger evidence of actual stored
content than relying only on the service's GET response, while remaining
author-operated and separate from proof of admission durability.

\begin{table}[tbp]
\centering
\footnotesize
\CloudIntervalTable
\caption{Mean worker-admission-to-dispatch latency and trial-cluster bootstrap confidence
intervals, in milliseconds. Repeated trials share the configured deployment;
these are not intervals over independent hosts or cloud regions.}
\label{tab:cloud-intervals}
\end{table}

\begin{figure}[tbp]
\centering
\includegraphics[width=\linewidth]{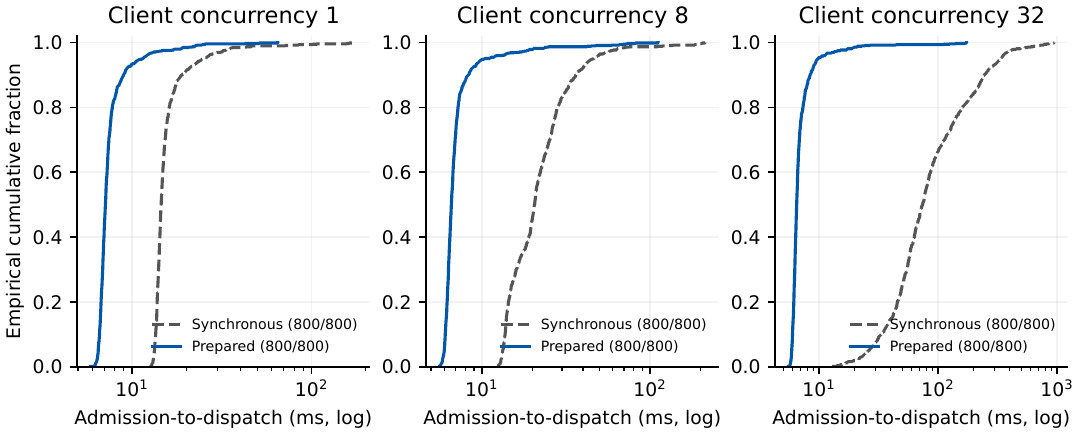}
\caption{Empirical worker-admission-to-dispatch latency distributions from
successful cloud records. Prepared readiness excludes already completed
review and minting; synchronous readiness includes them. Neither boundary
includes waiting to acquire its dispatch worker. Lifecycle includes preparation and
prepared-batch dwell in both the tables and Figure~\ref{fig:lifecycle}.
These curves do not represent complete-lifecycle speedup.}
\label{fig:latency}
\end{figure}

\begin{figure}[tbp]
\centering
\includegraphics[width=\linewidth]{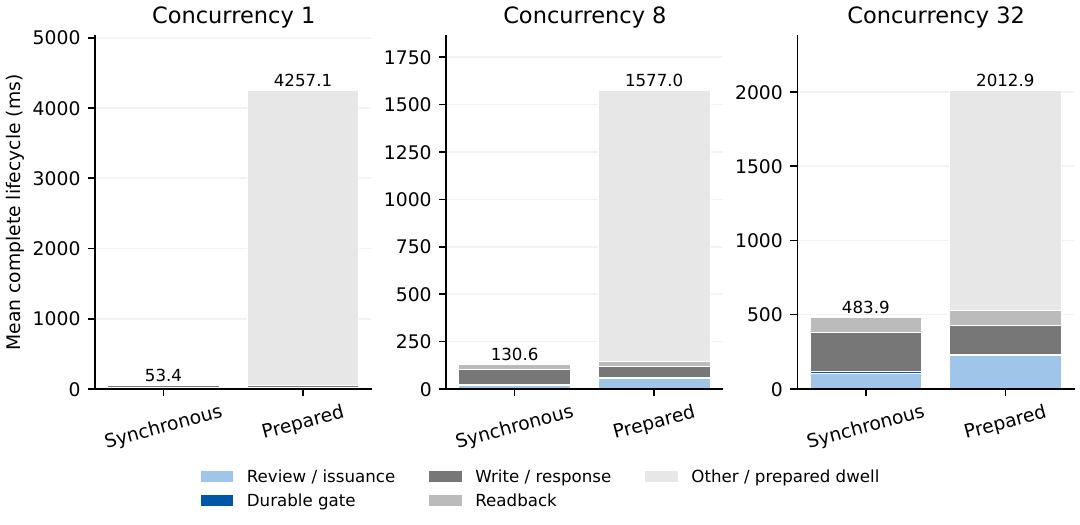}
\caption{Mean complete-lifecycle decomposition from the cloud records.
Prepared ready-to-dispatch excludes preparation, but the lifecycle shown here
includes review/minting, batch dwell and subsequent waiting, admission,
effect, and verification. Prepared lifecycle is longer despite its shorter
dispatch boundary. Component means, not phase percentiles,
are stacked.}
\label{fig:lifecycle}
\end{figure}

\FloatBarrier
\begin{table}[H]
\centering
\footnotesize
\SemanticResultsTable
\caption{Semantic outcomes against authored binary expectations, with
per-method denominators and errors. These observations are not independent
human labels, escalation accuracy, or cross-runtime generalization.
Method definitions appear in Table~\ref{tab:methods}.}
\label{tab:semantic}
\end{table}

In the source-pinned semantic artifact, the core matches the authored binary
expectations on every evaluated input, without an error outcome being
credited as a correct rejection. Table~\ref{tab:semantic} reports the exact
denominators and comparison-method outcomes. This supports conformance on
the exercised fixture contract. It does not validate the completeness of
the authors' expectations or establish Proposition~\ref{prop:preservation}'s
freshness and mediation premises in a deployed environment. Omission-method
disagreements are evidence about the omitted checks on these scenarios,
not measured failure rates of mature authorization products.

\begin{table}[H]
\centering
\footnotesize
\FaultResultsTable
\caption{Recorded outcomes for \FaultCases{} controlled fault and concurrency
cases. Durable admission, receipt recovery, replay rejection, and observed
external effects are distinct assertions. ``Absent (HTTP 404)'' denotes
absence of a service record, not an independent Azure absence check.
Local-key phases and deliberate
acknowledgement discard are not remote coherence or arbitrary network-fault
experiments.}
\label{tab:fault}
\end{table}

Table~\ref{tab:fault} separates admission from delivery: a crash can consume
authority without an effect. Successful blob retries do not establish
exactly-once execution or idempotence for arbitrary APIs.

\section{Scale, Preparation, and Deployment Decisions}
\label{scale-deployment}

The results establish a placement tradeoff, not a fleet-capacity result.
Prepared execution shortened the measured dispatch boundary while increasing
mean complete lifecycle latency at each tested concurrency in the frozen
study. Scaling this arrangement therefore requires asking which work
deserves preparation, what happens
while requests wait, and who prevents an unapproved effect. This section
develops three \emph{proposed extensions}: action-stream classification,
finite-capacity routing, and adaptive preparation support. These three
mechanisms are not implemented in the frozen study. Their diagrams specify
design obligations rather than measured recognition, queueing, or latency
performance.

\subsection{Classifying an action stream without transferring authority}
\label{scale-classification}

An agent trace mixes intentions, observations, dependencies, and effects.
Counting every trace item as a fresh authorization request wastes review
effort; treating the whole trace as one approved task risks concealing an
additional effect. The useful classification is consequently about the
operation's semantics, not its tool name, conversational grouping, or
position in a workflow. Figure~\ref{scale-fig-stream} distinguishes five
cases that require different handling.

\begin{figure}[tbp]
\centering
\begin{tikzpicture}[
  box/.style={zg node,align=center,font=\small,
    inner sep=5pt,minimum height=0.90cm},
  flow/.style={zg flow}]
  \node[box,text width=13.7cm,fill=zgBlue] (observe) at (0,0)
    {Observed runtime stream: classify each item by its actual operation};
  \node[box,text width=3.7cm] (context) at (-4.8,-1.65)
    {Context only};
  \node[box,text width=8.0cm] (contextout) at (2.15,-1.65)
    {Attach provenance; no execution authority};
  \node[box,text width=3.7cm,fill=zgBlue] (read) at (-4.8,-2.95)
    {Read evidence};
  \node[box,text width=8.0cm,fill=zgBlue] (readout) at (2.15,-2.95)
    {Authorize the read if required; bind source and freshness};
  \node[box,text width=3.7cm,fill=zgBlue] (dependent) at (-4.8,-4.25)
    {Dependent request};
  \node[box,text width=8.0cm,fill=zgBlue] (dependentout) at (2.15,-4.25)
    {Resolve prerequisites;\\each distinct effect needs its own pass};
  \node[box,text width=3.7cm,fill=zgCream] (new) at (-4.8,-5.55)
    {New side effect};
  \node[box,text width=8.0cm,fill=zgCream] (newout) at (2.15,-5.55)
    {Construct exact action;\\approve independently or refuse};
  \node[box,text width=3.7cm,fill=zgPeach] (ambiguous) at (-4.8,-6.85)
    {Ambiguous evidence};
  \node[box,text width=8.0cm,fill=zgPeach] (ambiguousout) at (2.15,-6.85)
    {Hold for clarification or review;\\no inferred permission};
  \draw[flow] (observe.south) -- (0,-0.80) -- (-7.2,-0.80)
    -- (-7.2,-6.85) -- (ambiguous.west);
  \draw[flow] (-7.2,-1.65) -- (context.west);
  \draw[flow] (-7.2,-2.95) -- (read.west);
  \draw[flow] (-7.2,-4.25) -- (dependent.west);
  \draw[flow] (-7.2,-5.55) -- (new.west);
  \draw[flow] (context) -- (contextout);
  \draw[flow] (read) -- (readout);
  \draw[flow] (dependent) -- (dependentout);
  \draw[flow] (new) -- (newout);
  \draw[flow] (ambiguous) -- (ambiguousout);
\end{tikzpicture}
\caption{Proposed action-stream classification, not an evaluated classifier.
Rows describe handling obligations, not observed proportions. A group label
cannot authorize a piggybacked action; shared review never merges nonces.}
\label{scale-fig-stream}
\end{figure}
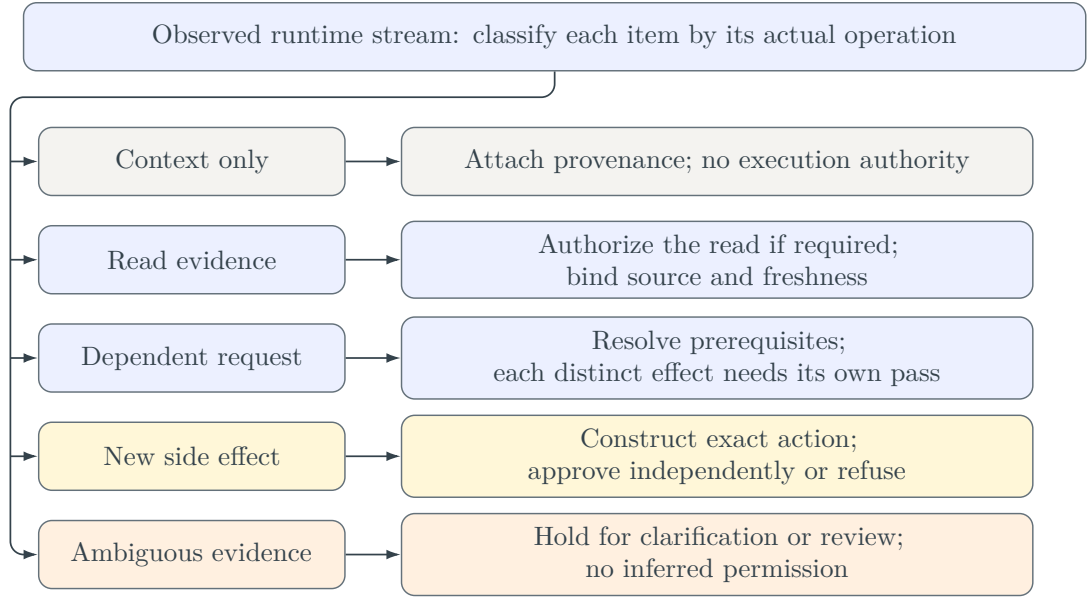

Context may explain a decision without crossing an execution boundary.
Read evidence is different: obtaining it can require authorization, disclose
data, or invoke a service with hidden effects. Calling an operation a read
does not exempt it from mediation. A dependent request may share inputs
with an earlier action, but dependency is not delegated authority. For
example, preparing a deployment description does not authorize the later
deployment, and approving that deployment does not authorize a notification
to a new external recipient. Each distinct effect requires its own exact
action binding and consumption identity. If downstream bytes depend on an
unresolved result, final pass issuance must wait until they are known.

Reviewers could share evidence acquisition or inspect related requests
together, provided each outcome remains separately attributable. Such batching
must not merge nonces, pool approvals into a reusable class credential, or
allow an extra operation to piggyback on an approved group label. Ambiguity
requires clarification, unsupported status, or refusal rather than an
optimistic classification. Evaluation would need independently adjudicated
operation boundaries, including compound commands and concealed mutations;
parser coverage alone would not measure false allows.

A compound operation also needs an explicit boundary definition. If a
destination genuinely exposes one atomic request containing several changes,
its canonical action must describe all policy-relevant changes and the
issuer must approve that whole operation. This does not permit an adapter
to collapse independently dispatchable calls into one action merely for
convenience. Classification records should retain the original request,
derived action, dependency links, and uncertainty reason so an auditor can
reconstruct the separation. Evidence may be reused when its provenance and
validity permit; execution authority is not inherited from that reuse.

\subsection{Finite capacity after safety eligibility}
\label{scale-capacity}

Even correctly separated requests compete for issuer workers, review time,
resource locks, durable storage, and destination capacity. The proposed
router first determines which next steps are admissible: gather evidence,
seek review, attempt admission, or terminate. Eligibility to enter an
admission queue is provisional, not permission to execute.
Figure~\ref{scale-fig-routing} places optimization between this classification
and the final safety check. Waiting can invalidate evidence, so the gate
must revalidate after scheduling and release execution only after durable
admission.

In this proposal, \emph{review work} means obtaining a fresh authorization
decision or the evidence it needs; that work may be automated or involve a
person. It is distinct from the gate's \Review{} outcome, which permits no
execution (Section~\ref{sec:contract}). The frozen cloud issuer uses an
automated fixture policy, not a human reviewer queue.

\begin{figure}[tbp]
\centering
\begin{tikzpicture}[
  box/.style={zg node,align=center,font=\small,
    text width=3.85cm,minimum height=1.02cm,inner sep=5pt},
  flow/.style={zg flow}]
  \node[box,text width=13.7cm,fill=zgCream] (safety) at (0,0)
    {Safety eligibility: exact action, authority, evidence, dependencies};
  \node[box,fill=zgPeach] (review) at (-4.8,-1.9)
    {$Q_{\rm review}$: bounded queue\\Missing or changed evidence};
  \node[box,fill=zgBlue] (ready) at (0,-1.9)
    {$Q_{\rm ready}$: bounded queue\\Candidate for admission};
  \node[box,fill=zgRose] (stop) at (4.8,-1.9)
    {Refuse / unsupported\\No execution path};
  \node[box] (reviewer) at (-4.8,-3.7)
    {Review scheduler\\Capacity and aging;\\review deadlines};
  \node[box] (scheduler) at (0,-3.7)
    {Resource scheduler\\Capacity and prerequisites};
  \node[box,fill=zgBlue] (return) at (-4.8,-5.5)
    {Fresh decision\\and evidence\\Return to eligibility};
  \node[box,fill=zgCream] (gate) at (0,-5.5)
    {Final gate revalidation\\Nonce + quota + receipt};
  \node[box,fill=zgRose] (nonexecute) at (0,-7.4)
    {Review / deny\\No dispatch};
  \node[box,fill=zgMint] (dispatch) at (4.8,-7.4)
    {Attempt external effect\\Then verify outcome};
  \draw[flow] (safety.south) -- (0,-0.90) -- (-4.8,-0.90)
    -- (review.north);
  \draw[flow] (safety.south) -- (ready.north);
  \draw[flow] (0,-0.90) -- (4.8,-0.90) -- (stop.north);
  \draw[flow] (review) -- (reviewer);
  \draw[flow] (ready) -- (scheduler);
  \draw[flow] (reviewer) -- (return);
  \draw[flow,dashed] (return.west) -- (-7.45,-5.5)
    -- (-7.45,0) -- (safety.west);
  \draw[flow] (scheduler) -- (gate);
  \draw[flow] (gate) -- (nonexecute);
  \draw[flow] (gate.east) -- (4.8,-5.5) --
    node[font=\scriptsize,align=right,left=3pt]
      {execute +\\successful commit} (dispatch.north);
  \node[font=\scriptsize,align=center,text width=13.7cm] at (0,-8.55)
    {Schedulers order eligible work; they grant no authority.\\
     Overflow or deadline expiry requires an explicit non-executing outcome.};
\end{tikzpicture}
\caption{Proposed finite-capacity routing. The cream-colored eligibility and
gate boxes are safety decisions;
schedulers optimize only the permitted next steps. Queue lengths, service
times, and deadline policies are deployment inputs, not measured constants.
The final gate may refuse a queued candidate.}
\label{scale-fig-routing}
\end{figure}

An illustrative discrete-slot model makes this separation explicit. Let
$x_{irt}$ select the next step for request $i$ on route $r$ in slot $t$;
$e_{irt}$ indicates that this step is currently eligible. Let $d_{irk}$ be
its demand on resource $k$ and $C_{kt}$ the available capacity. A necessary
feasibility fragment is
\begin{equation}
\begin{aligned}
 x_{irt}&\in\{0,1\}, & x_{irt}&\leq e_{irt},\\
 \sum_{r,t}x_{irt}&\leq 1\quad\text{for each pending step }i,
 &\sum_{i,r}d_{irk}x_{irt}&\leq C_{kt}\quad\text{for every }k,t.
\end{aligned}
\label{scale-eq-capacity}
\end{equation}
Multi-slot work requires occupancy constraints over its full duration;
dependencies require precedence constraints. Resource-specific exclusion
must cover all conflicting operations, not just a global worker count.
The model neither reserves policy quota nor proves authorization: quota
consumption remains an atomic gate obligation. An objective could penalize
waiting, missed review deadlines, and discarded preparation, but no objective
weight may relax eligibility.

Slow review needs an explicit fairness policy. Earliest-deadline ordering
alone can starve requests without urgent labels. Per-tenant service shares,
aging, and capacity reserved for older ambiguous cases are candidate
controls, not demonstrated guarantees. When deadlines become infeasible,
the system must defer with notice, reject, or escalate ownership; expiration
must never become automatic approval. Bounded queues also require overflow
rules and cancellation propagation. Offered arrivals, queued requests,
admission attempts, durable admissions, and completed effects are separate
counts. A closed-loop concurrency test cannot establish stability under an
independent offered load, particularly when preparation changes arrivals.

A capacity study should vary burstiness, review duration, resource skew,
and cancellations as well as average arrival rate. If offered resource
demand persistently exceeds available service, ordering alone cannot avoid
backlog growth: admission control or additional capacity is required.
Backpressure may improve the experience of accepted work while rejecting
more arrivals, so both outcomes must remain visible. Likewise, a planned
resource slot is only a scheduling decision. If correctness depends on
exclusive resource access, a lock, conditional operation, or equivalent
enforcement must retain that exclusion through the relevant effect;
observing an uncontended queue is insufficient.

\subsection{Active support from observable state}
\label{scale-support}

A single metaphorical measure of workload difficulty cannot determine how
much preparation is safe. The proposed support controller instead observes
pass expiry, the freshness mechanism's justified horizon, contention, and
review deadlines. For an already issued pass, if $T_i^{\rm exp}$ is its signed expiry and
$T_i^{\rm fresh}$ the end of a justified evidence-validity interval, define
remaining slack
\[
 s_i(t)=\min(T_i^{\rm exp},T_i^{\rm fresh})-t-\Delta_i,
\]
where $\Delta_i$ is the deployment's allowance for clock uncertainty and
remaining admission work. A copied timestamp or matching digest cannot
establish $T_i^{\rm fresh}$; leases require actual stability obligations
\citep{rwLeases1989}. Without such justification, freshness slack is unknown.
Queue-delay estimates guide resource allocation but are not safety bounds.
The frozen adapter checks validity at the post-lock admission-check sample,
not at commit return or eventual execution.

Before a pass exists, there is no signed $T_i^{\rm exp}$. Preparation
planning must instead use a prospective window permitted by issuer policy
and justified evidence, not an assumed future grant. Once issued, the
controller uses the actual signed expiry and reassesses the remaining wait.
A planning window grants no authority, and postponing a request never
extends an existing pass. If the issuer does not grant the planned window,
the schedule must change or the action must remain unexecuted.

Figure~\ref{scale-fig-support} maps these observations to proposed decisions.
Unknown or non-positive slack precludes relying on prepared evidence;
positive slack must still accommodate the planned wait. Before issuance,
contention can therefore postpone minting; after issuance, it can require
reassessment or replacement. An approaching review deadline preempts
optional preparation. Scheduling cannot substitute a predicted delay for
final revalidation.

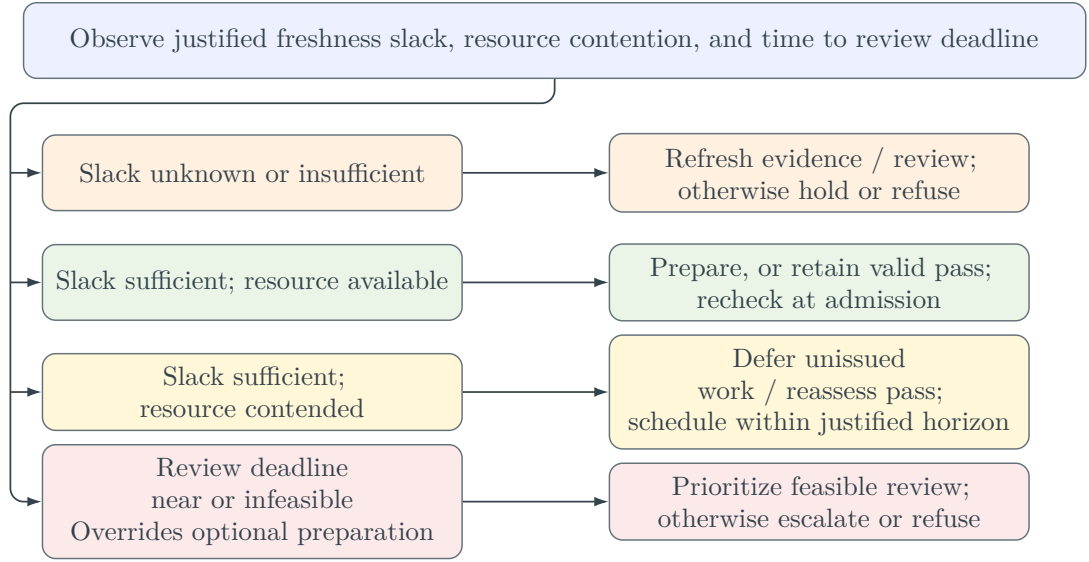
\begin{figure}[tbp]
\centering
\begin{tikzpicture}[
  box/.style={zg node,align=center,font=\small,
    inner sep=5pt,minimum height=1.00cm},
  flow/.style={zg flow}]
  \node[box,text width=13.7cm,fill=zgBlue] (observe) at (0,0)
    {Observe justified freshness slack, resource contention,
     and time to review deadline};
  \node[box,text width=5.2cm,fill=zgPeach] (stale) at (-4.0,-1.75)
    {Slack unknown or insufficient};
  \node[box,text width=5.2cm,fill=zgPeach] (staleout) at (3.5,-1.75)
    {Refresh evidence / review; otherwise hold or refuse};
  \node[box,text width=5.2cm,fill=zgMint] (free) at (-4.0,-3.20)
    {Slack sufficient; resource available};
  \node[box,text width=5.2cm,fill=zgMint] (freeout) at (3.5,-3.20)
    {Prepare, or retain valid pass;\\recheck at admission};
  \node[box,text width=5.2cm,fill=zgCream] (busy) at (-4.0,-4.65)
    {Slack sufficient;\\resource contended};
  \node[box,text width=5.2cm,fill=zgCream] (busyout) at (3.5,-4.65)
    {Defer unissued work / reassess pass;\\schedule within justified horizon};
  \node[box,text width=5.2cm,fill=zgRose] (deadline) at (-4.0,-6.10)
    {Review deadline near or infeasible\\Overrides optional preparation};
  \node[box,text width=5.2cm,fill=zgRose] (deadlineout) at (3.5,-6.10)
    {Prioritize feasible review; otherwise escalate or refuse};
  \draw[flow] (observe.south) -- (0,-0.83) -- (-7.2,-0.83)
    -- (-7.2,-6.10) -- (deadline.west);
  \draw[flow] (-7.2,-1.75) -- (stale.west);
  \draw[flow] (-7.2,-3.20) -- (free.west);
  \draw[flow] (-7.2,-4.65) -- (busy.west);
  \draw[flow] (stale) -- (staleout);
  \draw[flow] (free) -- (freeout);
  \draw[flow] (busy) -- (busyout);
  \draw[flow] (deadline) -- (deadlineout);
\end{tikzpicture}
\caption{Proposed active-support decision map linking observed conditions
to preparation, review and scheduling choices. Before issuance the horizon
is a policy-constrained planning window; afterward it uses actual signed
expiry. ``Sufficient'' means justified for the planned step, not a fixed timing budget.
Deadline priority cannot override freshness,
authority, or final revalidation. No benefit is estimated by this map.}
\label{scale-fig-support}
\end{figure}

Preparation is most plausible when final bytes become stable before their
release, review depends on slowly changing facts, and the action is likely
to be used. A scheduled immutable upload is a better candidate than a
mutation whose recipient, amount, or payload is selected from a preceding
response. It can hurt when predictions branch, users cancel, policies churn,
or durable admission is already the bottleneck. Earlier minting then adds
expiry and replacement work while competing with requests ready to proceed.

Prefetch should have a bounded horizon, an outstanding-work limit, and
cancellation on relevant change. Operators should distinguish warming
connections or fetching evidence from issuing authorization. A scheduler
reservation is not a pass, and speculative preparation must not perform
the external effect. Track unused, expired, superseded, and cancelled
preparations together with their compute, network, and review costs.
Changing the horizon is an experiment in useful preparation versus wasted
work, not permission to lengthen validity without a security justification.

For an immutable upload, an operator could first prepare only after the
body and object identity are fixed, then compare several bounded horizons
against synchronous approval. The comparison should record whether a pass
was actually consumed before relevant evidence changed, not merely whether
minting succeeded. For a workflow with unresolved downstream choices, the
safer experiment warms transport and retrieves permitted evidence while
delaying the action-bound signature. These are different interventions and
should be reported separately. A useful horizon for one action family does
not establish a default for payments, identity changes, or deployment
operations with different freshness requirements.

\subsection{From observation to complete mediation}
\label{scale-mediation}

Shadow deployment records reconstructed actions and hypothetical decisions
without preventing dispatch. It can reveal missing adapters, mapping errors,
and operational friction, but cannot demonstrate prevented effects. Enforced
deployment must mediate every in-scope route before the effect, including
retries, alternate credentials, direct SDK calls, shell subprocesses, and
background workers. Final bytes and authenticated caller context must remain
bound through dispatch; the agent must not possess an alternative route
around the gate. This is the complete-mediation requirement, not a property
obtained by adding a log entry \citep{rwSaltzer1975}.

The research interface does not meet this requirement against a hostile
bearer holder. Its trusted runner controls the privileged token, and the
effect service checks a caller-supplied receipt binding rather than
independently proving gate persistence. Holder fields are not proof of
possession. Deployment would need exclusive trusted credential custody and
non-bypassable dispatch, or receiver-verifiable admission bound to the
request and consumption domain. Neither hardening is established by the
current API. Key trust, snapshot administration, clock behavior, and storage
rollback protection also need named operational owners.

Migration from shadow to enforcement should therefore test the deployment
topology, not just switch a decision flag. An operator should attempt a
direct destination call from the agent's actual credential context, stop
the gate, and verify that alternate workers cannot dispatch the protected
operation. Observation-only adapters must remain visibly labeled, and
uncovered routes must not be included in an enforcement-coverage claim.
An emergency path can be legitimate, but needs separately authenticated
authority, recorded use, and an explicit scope; a fail-open fallback hidden
inside normal retry logic defeats the boundary being evaluated.

Scaling consumption across hosts requires a coherent single-use authority
or a proved ownership partition with fencing and migration rules. Copying
SQLite databases cannot provide global nonce or aggregate-quota enforcement.
Recovery likewise needs an explicit delivery protocol: durable admission
may precede a crash with no effect, while a lost response may hide a completed
effect. An outbox can retain dispatch intent, but destination idempotency
and reconciliation remain necessary \citep{rwAWSOutbox}. The reference
receipt is not an implemented autonomous outbox. Policies requiring
execution-time validity need receiver checks or another justified protocol.

\subsection{Responsibility and comparative deployment choices}
\label{scale-responsibility}

The policy owner specifies allowed operations, dependencies, and acceptable
failure behavior. The issuer implements review and protects signing keys.
The adapter owner establishes faithful observations and coverage; the gate
operator maintains consumption and receipt durability. The destination owner
defines conditional execution and retry semantics. Operations owns queue
limits, incident handling, and recovery; an independent evaluator checks the
acceptance evidence. These responsibilities can share an organization, but
should not disappear into a general claim that governance is enabled.

Earlier OSuite architecture descriptions and comparisons with vendor-native
governance are design context, not actual product benchmarks in this study.
The fixture comparators do not measure commercial platforms. Native controls
may already provide appropriate request binding, identity integration, and
operational support; an embedded policy engine can also avoid remote review
\citep{rwOPAPerformance,rwCedar2024}. An additional pass layer is justified
only if its integration and evidence benefits outweigh duplicated controls
and operational cost for the selected topology. Cross-runtime portability
requires adapter-specific validation, not assumed identical fingerprints.
A buyer may reasonably choose native controls alone, a complementary gate,
or no deployment. Universal vendor superiority is neither necessary nor
supported.

\subsection{Buyer acceptance and go/no-go evidence}
\label{scale-acceptance}

A pilot should begin with a named action inventory, independent expected
decisions, agreed failure limits, and a rollback owner. Capture a before-state
baseline using the existing controls, then an after-state under the candidate
configuration with the same workload mix and offered-arrival protocol.
Separate shadow and enforced results. Record configuration and policy
changes so apparent improvements are not attributed to preparation when
coverage or obligations changed. The following tests are acceptance gates,
not outcomes reported here.

Before collecting results, the buyer should specify which populations may
benefit and which may not regress: interactive callers, background jobs,
reviewers, and destination operators can experience different costs.
Retain per-action arrival and terminal-state timestamps, including attempts
that never acquire a worker. Fix the observation window and disposition of
unfinished requests so draining a backlog outside the window does not
create an apparent improvement. A favorable average with unacceptable
slow-review tails, incomplete destination reconciliation, or increased
operator intervention should not pass merely because the prepared boundary
looks faster.

\begin{enumerate}
  \item \textbf{Coverage and unauthorized effects.} Enumerate actual
  side-effect paths and reconcile them with gate observations and destination
  records. Report action coverage over all in-scope attempts, unsupported
  cases, and independently adjudicated false allows over expected-not-allow
  attempts, with errors and uncertainty visible. Test bypass, payload
  substitution, stale evidence, and replay. Any unexplained bypass or false
  allow is a no-go pending investigation; zero observed cases is not proof
  of universal safety.

  \item \textbf{User-visible waiting and review fairness.} Measure
  queue-inclusive p95 and p99 from externally recorded arrival to decision
  and separately to verified effect. Retain failures, cancellations, and
  timeouts; report successful latency conditionally rather than assigning
  failures zero duration. Include preparation, dwell, review wait, lock wait,
  and recovery. Compare tenants and slow-review classes for starvation and
  deadline misses. Do not add phase percentiles or infer an SLA from the
  worker-admission boundary.

  \item \textbf{Bounded failure and recoverability.} Exercise issuer loss,
  stale snapshots, storage failure, queue overflow, process crashes, ambiguous
  acknowledgements, and clock anomalies. Verify time to refusal, maximum
  outstanding work, retained intent, and recovery ownership against limits
  agreed before testing. Reconcile before-state and after-state at the
  destination, not just receipt presence. Unbounded retry, unexplained
  effects, or rollback that re-enables consumed authority blocks rollout.

  \item \textbf{Useful preparation and operational economics.} Report
  prepared actions used before expiry, wasted preparation by cause, repeated
  reviews, and issuer and storage demand. Compare cost per verified completed
  action and total workload cost, including review labor, infrastructure,
  integration, incident handling, and audit effort. Include the cost of
  rejected and abandoned work. No customer savings follow from a shorter
  boundary alone; the buyer must decide whether the observed tradeoff meets
  its own cost and response requirements.
\end{enumerate}

Expansion should proceed by action family and enforcement boundary only
after those gates pass. A failure may justify a narrower supported scope,
fresh synchronous review, or retaining the existing control rather than
weakening eligibility. This makes deployment a falsifiable operational
decision while keeping proposed scale mechanisms separate from the measured
reference system.

\section{Limitations and Deployment Implications}
\label{sec:limits}

\paragraph{The strongest guarantee depends on freshness not implemented here.}
The gate protects its local snapshot against per-request substitution, but
does not keep it synchronized with every policy authority or resource.
Configuration hashing is not state observation. Offline revocation remains
bounded only by whatever expiry, update, or lease policy a deployment
actually enforces. Clock trust and permitted skew must be specified as part
of that policy. A conservative deployment should route unsupported or
uncertain dependencies to fresh review; it should not advertise current
authorization solely because a local hash matches.

\paragraph{The trusted runner is part of the security boundary.}
Bearer credentials, holder fields, and caller-supplied receipts do not stop
a malicious authorized runner. Production use with untrusted presenters
would need authenticated observations and an enforcement design that binds
the receiver to verifiable admission, or prevents bypass by construction.
Key bootstrap, rotation, revocation distribution, administrative access to
SQLite, and rollback protection also remain deployment obligations. A signed
pass cannot protect a database that an adversarial host restores to a
pre-consumption snapshot.

\paragraph{Single-use exactness has a real lifecycle cost.}
Every changed action needs review and signing, and every admission needs
durable state. Prepared batching can lengthen the time between mint and use,
increase expiry risk, and make complete lifecycle slower even when dispatch
readiness looks better. The present cloud workload knows the intended body
in advance and does not quantify abandoned preparations, inaccurate
prediction, human-review latency, or a stream of interdependent agent
decisions. Neither the issuer nor the gate is relieved of total work merely
by moving it earlier.

\paragraph{One local domain is not a global replay service.}
SQLite serializes writes and supplies a practical local transaction boundary.
A busy writer can delay admission or cause fail-closed availability loss.
The study's concurrency points do not establish unbounded scale, arrival-rate
stability, or a service-level guarantee. Cross-region single use would require
a coordinated consumption authority, a correctness-preserving ownership
partition with migration rules, or a different protocol. Copying the same
pass and snapshot into independent regions is a direct counterexample.

\paragraph{Admission is not delivery or completion.}
The local receipt records what was authorized, not what an independent
service necessarily did. A crash can consume authority without performing
the action. Retrying delivery requires retained intent, bounded
idempotency semantics, and an explicit response to ambiguous outcomes.
The blob service's content checks and conditional creation support
the selected operation only. Policies that depend on execution-time state
need receiver-side checks or stronger coordination. This artifact does not
offer a generic distributed transaction or exactly-once external executor.

\paragraph{Evaluation breadth and independence are limited.}
Authored mutation cases exercise selected obligations rather than a natural
distribution of agent mistakes or adversarial traffic. Repeated domain
renderings are not new independent ground truth. Canonicalization is
evaluated through current project representations, whose fingerprints
include runtime; neither parser coverage nor cross-runtime semantic identity
is measured. The embedded baseline and omission variants are diagnostic
comparators, not production implementations of alternative systems. Cloud
results are from one configured deployment and storage backend, with
bounded concurrency, ordered cells, and alternating rather than fully
randomized mode order.

\paragraph{Verification is narrower than independent attestation.}
The new artifact verifier checks signatures, record structure, counts,
source hashes, decision replay, receipt links, and summary consistency.
It cannot establish that a supplied timestamp was honestly measured or that
an unpinned bundle was not replaced wholesale. The same source-defined
semantic workload is replayed; no independent policy annotator is introduced.
Separately conducted external-verifier and canary exercises described in
earlier project work are prior background only. They are not runs of this
revalidation experiment, are excluded from its denominators and tables, and
are not evidence that its current local artifact has been externally
replicated.

\section{Artifact and Reproducibility}
\label{sec:artifact}

The project home is
\url{https://github.com/OndCo/Agent-Action-Boundary-Benchmark}.
The arXiv source package includes
\path{anc/zerogate-revalidation-artifact.tar.xz} and its accompanying README.
This ancillary archive provides runnable code, frozen raw records, manifests,
verification reports, exact source, and run instructions. The source package
also includes the manuscript, bibliography, generated result include, and
empirical figures.

The semantic producer writes \texttt{records.jsonl} and
\texttt{manifest.json}, including seed, environment, input and record
checksums, source-file hashes, method definitions, counts, and summaries.
It checks source identity around the run rather than silently completing
against changed evaluator code. The cloud producer records configuration
and environment, each attempt and available timing, gate and service
audits, summaries, and checksums. The fault producer preserves per-case
inputs, subprocess outcomes, audits, and its overall report. New output
directories are required so reruns do not overwrite evidence.

\path{verify-artifact.mjs} independently reconstructs artifact summaries
and validates pass signatures without importing the durable gate's signature
checker. For semantics it also replays the checksum-verified local source
and compares input labels and actual outputs. It does not execute arbitrary
code supplied inside an untrusted data bundle. Externally pinned record
hashes or trusted issuer keys strengthen provenance; checksums stored only
alongside their own data detect inconsistency, not a coordinated replacement.
The supplement must preserve the exact source revision used by a manifest,
not just a later checkout with the same filenames.

\path{verify-cloud-store.mjs} and \path{verify-faults.mjs} additionally open
checkpointed SQLite files in immutable, read-only mode and reconcile all
five stored tables with the exported evidence. They reject nonempty journals
and inconsistent rows. This checks the retained database bytes, not a signed
attestation of historical fsync or physical power-loss durability.

Reproduction requires a Node.js runtime supporting the used SQLite API, an
operator-provisioned service and persistent gate directory, a bearer-token
file, and appropriate HTTPS trust when using a remote service. Reproducing
the cloud study additionally requires Azure storage permissions and managed
identity configuration. The local SQLite backend is useful for unit and
integration tests but is not a substitute for the Azure result. Private
keys, bearer tokens, and deployment secrets must not enter the public
supplement. The build depends on \texttt{revalidation-results.tex}; its macros
are generated from the frozen outputs rather than populated with placeholder
measurements.

\section{Conclusion}
\label{sec:conclusion}

\zerogate{} makes earlier exact-action approval usable at a later durable
admission boundary without claiming that approval work disappears.
Its preservation argument is conditional on complete dependencies, current
evidence, faithful mediation, and atomic consumption. The reference
implementation makes local nonce, quota, and receipt consistency concrete,
while exposing the remaining limits of stale snapshots, trusted bearer
APIs, and non-atomic external effects. The evaluation separates authored
semantic conformance from fault histories and actual cloud timing, and
charges prepared-batch dwell to lifecycle. These boundaries, rather than
an unconditional speed or exactly-once claim, define the contribution.

For an engineering organization, the practical opportunity is to align
authorization with the time at which its inputs become known, while retaining
a reliable final check at the point of consequence. The transit inspiration
helps explain that allocation of work; the action representation, authority
contract and admission transaction make it inspectable. The scale mechanisms
outlined here describe how to investigate useful preparation and finite
capacity without treating scheduling as a new source of permission. Their
implementation and evaluation remain separate work. A deployment should
advance when its own action coverage, evidence quality and complete user
experience improve, not merely when a local gate reports a smaller number.

\par\bigskip
\appendix
\section{Field Guide to the Evidence}
\label{app:field-guide}

This appendix is for a reader moving between the report and the downloadable
artifact. The field groupings explain the implemented profile; they are not
a replacement schema or a new interoperability standard. The complete
record includes null fields and configuration labels as issued. Those bytes
must be preserved for recomputation, even when a label is less descriptive
than the implementation discussed in the text.

\subsection{Locate the worked observation}

In the unpacked ancillary archive, \path{data/cloud/records.jsonl} contains
the worked observation \ReportCaseId{}. Find the record by its identifier,
not by an assumed line number. The body is 128 UTF-8 bytes and its
\texttt{payload\_bytes} field agrees with that length. The following values
are extracted by a presentation generator that refuses any cloud file whose
SHA-256 differs from the frozen record digest in this report.

\begin{table}[!htbp]
\centering\small
\begin{tabularx}{\linewidth}{L{0.21\linewidth}Y}
\toprule
Object & Complete value from the selected observation \\
\midrule
Operation ID & \ReportCaseId{} \\
Pass ID & \ReportCasePassId{} \\
Payload SHA-256 & \ReportCasePayloadHash{} \\
Action fingerprint & \ReportCaseFingerprint{} \\
Pass hash & \ReportCasePassHash{} \\
Receipt hash & \ReportCaseReceiptHash{} \\
Azure ETag & \ReportCaseETag{} \\
\bottomrule
\end{tabularx}
\caption{Lookup values, not a self-contained proof. Signature checking also
needs the complete pass, envelope and trusted public key. Content checking
needs the actual body, and receipt-store reconciliation needs the database.}
\label{tab:worked-identifiers}
\end{table}

The action fingerprint and payload hash are intentionally different. One
commits to a canonical object containing operation, destination, object,
runtime and payload metadata; the other commits only to the bytes. The pass
hash additionally binds the authority and validity conditions. The receipt
hash identifies the recorded admission statement. Equality is meaningful
between corresponding fields, not across every hash in the bundle. An ETag
is a destination-issued version identifier, not an alternative cryptographic
proof of the action's authorization.

\subsection{Read the pass as several obligations}

\begin{table}[H]
\centering\small
\begin{tabularx}{\linewidth}{L{0.28\linewidth}Y}
\toprule
Field group & Question it answers, and what it does not establish \\
\midrule
\path{issuer}, envelope & Which configured issuer signed this pass hash?
The trusted key mapping, not a caller-chosen key alone, supplies trust. \\
\path{holder} & Which agent, principal, workspace and session does the
grant name? Equality to trusted observations is not proof of possession. \\
\path{action_binding} & Which canonical action and final payload commitment
are approved? The observer must reconstruct the actual outgoing request. \\
\path{authority.granted} & Which operations, resources, destinations, systems
and constraints bound the grant? These do not override the exact-action
fingerprint with a reusable class permission. \\
\path{authority.requested_ref} & Which upstream request is referenced?
It is null in the selected cloud case; no skill-contract execution is implied. \\
\bottomrule
\end{tabularx}
\caption{Identity and authority fields of the implemented exact-action profile.}
\label{tab:field-identity}
\end{table}

The upstream authorizer, not this reference gate, is responsible for keeping
the grant within the original requested scope
(Section~\ref{foundation-authority-layers}). The gate checks the resulting
signed grant and exact action. A request reference alone does not demonstrate
that upstream check.

\begin{table}[H]
\centering\small
\begin{tabularx}{\linewidth}{L{0.28\linewidth}Y}
\toprule
Field group & Required interpretation at admission \\
\midrule
\path{policy_binding}, \path{state_binding} & Compare supported bindings
against gate-owned evidence. Agreement with stale configuration remains stale;
arbitrary current resource state is not observed by a configuration digest. \\
\path{revocation_binding} & Check the supported epoch against the local
snapshot. Absence of a newer epoch locally is not proof that none exists. \\
\path{validity} & Check the inclusive validity interval at the post-lock
assessment sample. This does not establish validity at later external execution. \\
\path{nonce}, \path{budget} & Enforce single use and applicable capacity in
one coordinated store. Different regional databases are different domains. \\
\path{evidence_profile}, \path{commit_profile} & Describe requested evidence
and routing obligations. Descriptive fields do not demonstrate their discharge. \\
\bottomrule
\end{tabularx}
\caption{Freshness, consumption and evidence fields. Read these with the
actual gate implementation, not just their names.}
\label{tab:field-residual}
\end{table}

For example, the issued pass's \texttt{receipt\_sink} label is
\texttt{local\_append\_only\_buffer}, inherited from the general profile.
The research adapter actually enforces its receipt obligation using SQLite
write-ahead logging (\texttt{journal\_mode=WAL}) and the separate
\texttt{synchronous=FULL} durability setting, recording the row in its
admission transaction. That label neither
makes the database append-only against administrators nor substitutes for
checking the retained store. The core's evidence flags express obligations
the adapter must discharge; presenting a boolean is not an independent
persistence attestation. Changing the signed label in a published record
would invalidate the original pass rather than improve its documentation.

The \texttt{risk} and \texttt{privacy\_profile} fields likewise do not turn
the fixture issuer into a validated risk classifier or a data-loss-prevention
engine. The cloud workload uses synthetic content. Any commercial claim
about sensitive-data detection requires its own input population, labels
and evaluation, separate from exact content binding.

\subsection{Separate the records by what they witness}

\begin{table}[!htbp]
\centering\small
\begin{tabularx}{\linewidth}{L{0.34\linewidth}Y}
\toprule
Artifact location & Evidence role \\
\midrule
\path{data/cloud/records.jsonl} & Per-attempt request, signed pass,
decision, timings, service result and readback. Failures remain attempts. \\
\path{data/cloud/gate.sqlite} & Checkpointed consumption and receipt state,
reconciled with the exported audit and attempt records. \\
\path{data/azure-direct-readback/} & Actual downloaded object bytes obtained
through the separate Azure CLI retrieval path. \\
\path{data/azure-direct-inventory.json} & Destination metadata used to
cross-check lengths and ETags, not an independent signed attestation. \\
\path{data/faults/} & Controlled histories, worker results and retained
state for race, mutation, expiry and recovery cases. \\
\path{data/semantic/} & Authored inputs and outcomes under the pinned
scenario generator; not completed independent human annotation. \\
\path{MANIFEST.json} & File sizes and digests for the packaged payload;
compare its digest with an externally obtained expected value. \\
\bottomrule
\end{tabularx}
\caption{Navigation inside the independently unpackable evidence archive.}
\label{tab:artifact-map}
\end{table}

There is a useful escalation of questions. A matching manifest establishes
package consistency. Signature checks establish issuer binding against the
selected trusted key. Decision replay establishes agreement with the pinned
implementation. Database reconciliation establishes consistency of the
retained persistent rows. Direct readback establishes that the archived
bytes match the requested content. None independently witnesses every event
in the historical execution, and combining them must not erase their trust
boundaries. The archived material is sufficient for these offline checks
without a Studio login or access to the original Azure account.

Section~\ref{sec:artifact} identifies the corrected evidence package; the
README accompanying \path{anc/zerogate-revalidation-artifact.tar.xz} gives
the full verification commands. For an
initial inspection, extract the ancillary archive to a new directory and
run its inventory verifier, then the cloud, semantic, database, fault and
readback checks. A fresh cloud experiment requires new infrastructure and
credentials; offline recomputation does not. The historical
\texttt{zerogate-v0.1} public release is not the corrected evidence manifest.

\section{Decision Checklist for Readers}
\label{app:reader-checklist}

A research or procurement discussion should end with an explicit boundary
of what has been accepted. The following questions make assumptions,
responsibilities and acceptance criteria inspectable.

\begin{enumerate}
  \item Can an independent party reconstruct the same final request from
  the artifact, including payload, destination and identity assumptions?
  If not, action-binding claims are incomplete even when hashes are present.
  \item Is the presented approval for one exact action or for a class of
  actions? If it is a class, where are the membership rule, attenuation and
  aggregate accounting defined? This report evaluates the exact-action case.
  \item Who can update policy/state observations and issuer keys, and what
  justifies their freshness at the required instant? Expiry alone is not a
  complete answer for policies requiring immediate revocation.
  \item What prevents a caller from bypassing the gate or replacing bytes
  after verification? Identify the actual credential and dispatch topology,
  not just an SDK hook in a diagram.
  \item Which durable operation consumes authority, and how does recovery
  distinguish unused, admitted and already executed requests? An admission
  receipt and a destination receipt should not be interchangeable.
  \item Does the performance comparison include the same security work,
  input population and observation window? Are preparation, waiting, errors
  and destination verification counted, and are the compared clocks equal?
  \item Which findings are measured here, which come from authored fixtures,
  and which are proposed designs? Unsupported runtime routes and unimplemented
  scheduling mechanisms must remain visible in the coverage statement.
\end{enumerate}

These questions can be used with a native runtime control, an embedded
policy engine, or a separate governance service. They do not assume that
every deployment needs another platform. They identify the evidence needed
to decide whether a prepared-action boundary improves a particular workflow
without weakening its authority or operational accountability.

{\small
\bibliographystyle{plainnat}
\bibliography{references,report-foundations}
}
\end{document}